\documentclass{article}
\usepackage[margin=1in]{geometry}

\usepackage[nospace,noadjust]{cite}
\usepackage{amsmath,amssymb,amsfonts,times}
\usepackage{graphicx}
\usepackage{textcomp}
\usepackage{xcolor}
\usepackage{subfigure}
\usepackage{epsfig}
\usepackage{multirow}
\usepackage{algorithm}
\usepackage{algorithmicx}
\usepackage{algpseudocode}
\usepackage{array}
\usepackage{epstopdf}
\usepackage{color}
\usepackage{diagbox}
\usepackage{bm}
\usepackage{tcolorbox}
\usepackage{pdfsync}
\usepackage{makecell}

\usepackage{url}
\usepackage[hidelinks]{hyperref}
\usepackage{amsmath,amsthm,amssymb,amsbsy}
\usepackage{paralist}
\usepackage{xcolor}
\usepackage{color}
\usepackage{graphicx}
\graphicspath{{./figs/}}
\usepackage{comment}
\usepackage{multirow}
\usepackage[toc, page]{appendix}

\usepackage{graphicx}
\usepackage{fancyhdr}
\usepackage{cite}
\usepackage{cleveref}

\newtheorem{lemma}{Lemma}
\newtheorem{definition}{Definition}

\newtheorem{theorem}{Theorem}

\theoremstyle{remark}

\theoremstyle{problem}

\renewcommand{\O}{\mathbb{O}}
\newcommand{\R}{\mathbb{R}}
\def \real    { \mathbb{R} }

\newcommand{\e}{\begin{equation}}
\newcommand{\ee}{\end{equation}}
\newcommand{\en}{\begin{equation*}}
\newcommand{\een}{\end{equation*}}
\newcommand{\eqn}{\begin{eqnarray}}
\newcommand{\eeqn}{\end{eqnarray}}
\newcommand{\bmat}{\begin{bmatrix}}
\newcommand{\emat}{\end{bmatrix}}

\DeclareMathAlphabet\mathbfcal{OMS}{cmsy}{b}{n}

\newcommand{\E}{\operatorname{\mathbb{E}}}

\newcommand{\vct}[1]{\boldsymbol{#1}}
\newcommand{\mtx}[1]{\boldsymbol{#1}}

\newcommand{\<}{\langle}
\renewcommand{\>}{\rangle}

\DeclareMathOperator*{\argmin}{\text{arg~min}}

\def \st {\operatorname*{s.t.\ }}

\newcommand{\wh}{\widehat}

\newcommand{\wt}{\widetilde}
\newcommand{\ol}{\overline}

\newcommand{\calA}{\mathcal{A}}
\newcommand{\calB}{\mathcal{B}}
\newcommand{\calC}{\mathcal{C}}
\newcommand{\calD}{\mathcal{D}}

\newcommand{\calH}{\mathcal{H}}
\newcommand{\calI}{\mathcal{I}}

\newcommand{\calM}{\mathcal{M}}

\newcommand{\calP}{\mathcal{P}}

\newcommand{\calS}{\mathcal{S}}

\newcommand{\calX}{\mathcal{X}}
\newcommand{\calY}{\mathcal{Y}}
\newcommand{\calZ}{\mathcal{Z}}

\newcommand{\vh}{\vct{h}}

\newcommand{\vr}{\vct{r}}

\newcommand{\vx}{\vct{x}}
\newcommand{\vy}{\vct{y}}

\newcommand{\mA}{\mtx{A}}
\newcommand{\mB}{\mtx{B}}

\newcommand{\mD}{\mtx{D}}
\newcommand{\mE}{\mtx{E}}

\newcommand{\mG}{\mtx{G}}

\newcommand{\mR}{\mtx{R}}

\newcommand{\mU}{\mtx{U}}
\newcommand{\mV}{\mtx{V}}

\newcommand{\mX}{\mtx{X}}
\newcommand{\mY}{\mtx{Y}}
\newcommand{\mZ}{\mtx{Z}}

\newcommand{\mSigma}{\mtx{\Sigma}}

\newcommand{\mId}{{\bf I}}

\graphicspath{{./figs/}}

\newlength{\imgwidth}
\newboolean{twoColVersion}
\setboolean{twoColVersion}{false}
\newcommand{\twoCol}[2]{\ifthenelse{\boolean{twoColVersion}} {#1} {#2} }

\title{\LARGE \bf A Scalable Factorization Approach for \\ High-Order Structured Tensor Recovery}

\author{Zhen Qin, Michael B. Wakin and Zhihui Zhu\thanks{ZQ (email: qin.660@osu.edu) and ZZ (email: zhu.3440@osu.edu)  are with the Department of Computer Science and Engineering, the Ohio State University;  and MBW (email: mwakin@mines.edu) is with the Department of Electrical Engineering, Colorado School of Mines.}
}

\begin{document}

\maketitle

\begin{abstract}

Tensor decompositions, which represent an $N$-order tensor using approximately $N$ factors of much smaller dimensions, can significantly reduce the number of parameters. This is particularly beneficial for high-order tensors, as the number of entries in a tensor grows exponentially with the order. Consequently, they are widely used in signal recovery and data analysis across domains such as signal processing, machine learning, and quantum physics. A computationally and memory-efficient approach to these problems is to optimize directly over the factors using local search algorithms such as gradient descent, a strategy known as the factorization approach in matrix and tensor optimization. However, the resulting optimization problems are highly nonconvex due to the multiplicative interactions between factors, posing significant challenges for convergence analysis and recovery guarantees. Existing studies often provide suboptimal guarantees with respect to tensor order $N$. For instance, the guarantee for Tucker decomposition appears to deteriorate exponentially with $N$, while the guarantee for tensor-train decomposition in sensing problems requires a Restricted Isometry Property (RIP) that scales with $N$. Moreover, most existing work primarily focuses on individual tensor decompositions and relies on case-by-case analysis.

In this paper, we present a unified framework for the factorization approach to solving various tensor decomposition problems. Specifically, by leveraging the canonical form of tensor decompositions—where most factors are constrained to be orthonormal to mitigate scaling ambiguity—we apply Riemannian gradient descent (RGD) to optimize these orthonormal factors on the Stiefel manifold. Under a mild condition on the loss function, we establish a Riemannian regularity condition for the factorized objective and prove that RGD converges to the ground-truth tensor at a linear rate when properly initialized. Notably, both the initialization requirement and the convergence rate scale polynomially rather than exponentially with $N$, improving upon existing results for Tucker and tensor-train format tensors.

\end{abstract}

\section{Introduction}
\label{introduction}

As generalizations of matrices to higher dimensions, tensors play a crucial role in a wide range of scientific and engineering domains, including applications such as system identification \cite{liu2010interior, zhang2021designing}, quantum state tomography \cite{francca2021fast, lidiak2022quantum,qin2024quantum}, collaborative filtering \cite{li2017mixture, yu2020recovery}, principal component analysis \cite{hotelling1933analysis}, biological sequencing data analysis \cite{faust2012microbial}, network analysis \cite{sewell2015latent}, image processing \cite{salmon2014poisson}, deep networks \cite{khrulkov2017expressive}, network compression or tensor networks \cite{stoudenmire2016supervised, yang2017tensor,novikov2015tensorizing, tjandra2017compressing, yu2017long, ma2019tensorized}, recommendation systems \cite{frolov2017tensor}, and more. However, for an $N$-order tensor $\calX\in\R^{d_1\times \cdots \times d_N}$, the size increases exponentially with respect to the order $N$, posing a significant challenge in storage and processing. To address this issue, tensor decomposition has emerged as a popular approach for compactly representing tensors. Three commonly used tensor decompositions are the Tucker \cite{Tucker66}, tensor train (TT) \cite{Oseledets11} and canonical polyadic (CP) \cite{Bro97} decompositions.

\paragraph{Tucker decomposition}
In order to efficiently achieve a compact representation of a tensor with a low multilinear rank, the Tucker decomposition method has gained widespread adoption. Specifically, Tucker decomposition is commonly used to compress data into a smaller tensor (known as the core tensor) and identify the subspaces spanned by the fibers (represented by the column spaces of the factor matrices). Tucker decomposition represents the $(s_1,\dots,s_N)$-th element of $\calX$ as
\begin{eqnarray}
    \label{Thr element of Tucker tensor in intro}
    \calX(s_1,\dots,s_N)=\sum_{i_1=1}^{r_1^\text{tk}}\cdots\sum_{i_N=1}^{r_N^\text{tk}} \mU_1(s_1,i_1)\cdots \mU_N(s_N,i_N)\calS(i_1,\dots,i_N),
\end{eqnarray}
where $\mU_i\in\R^{d_i\times r_i^\text{tk}}$, $i\in[N]$ are column-wise orthonormal factor matrices and $\calS\in\R^{r_1^\text{tk}\times r_2^\text{tk}\times \cdots \times r_N^\text{tk}}$ is a core tensor. The $N$-tuple $(r_1^\text{tk},\dots,r_N^\text{tk})$ denotes the multilinear rank of the tensor in the Tucker format. We define the compact form of the Tucker tensor as follows:
\begin{eqnarray}
    \label{The compact form of Tucker tensor in intro}
    \calX= [\![\calS; \mU_1, \mU_2,\dots, \mU_N ]\!],
\end{eqnarray}
which reduces the number of parameters from $O(\ol d^N)$ to  $O(N\ol d \ol r^\text{tk} + {\ol r^\text{tk}}^N)$ for $\ol d = \max_{i}d_i$ and $\ol r^\text{tk} = \max_{i}r_i^\text{tk}$.
Given any tensor, a higher-order singular value decomposition (HOSVD) \cite{sidiropoulos2000uniqueness} can be employed to find the Tucker decomposition with quasi-optimality, facilitating various applications such as neuroimaging analysis \cite{li2017parsimonious,li2018tucker}, imaging and video processing \cite{das2021hyperspectral}, and longitudinal relational data analysis \cite{hoff2015multilinear}.

\paragraph{Tensor train (TT) decomposition} To further reduce the number of parameters, TT decomposition provides an efficient representation  with only linear complexity in $N$ and has been widely adopted in large-scale quantum state tomography \cite{cramer2010efficient,lanyon2017efficient,wang2020scalable,verstraete2004matrix,pirvu2010matrix,werner2016positive,jarkovsky2020efficient}, deep networks \cite{khrulkov2017expressive}, network compression  \cite{stoudenmire2016supervised, novikov2015tensorizing, yang2017tensor, tjandra2017compressing, yu2017long, ma2019tensorized} and more.
 Specifically, TT decomposition represents the $(s_1,\dots,s_N)$-th element of $\calX$ as
\begin{eqnarray}
    \label{The element of TT in intro}
    \calX(s_1,\dots,s_N)=\prod_{i=1}^N{\mX_i(:,s_i,:)}=\prod_{i=1}^N{\mX_i(s_i)},
\end{eqnarray}
where tensor factors $\mX_i \in\R^{r_{i-1}^\text{tt}\times d_i \times r_i^\text{tt}}, i\in[N]$ with $r_0^\text{tt}=r_N^\text{tt}=1$. In \eqref{The element of TT in intro}, we denote $\mX_{i}(:,s_i,:)$---one slice of $\mX_i$ with the second index being fixed as $s_i$---by  ${\mX_i(s_i)}\in\R^{r_{i-1}\times r_{i}}$ since it will be extensively used.
To simplify the notation, we write the TT format with factors $\{\mX_i\}$ as
\begin{eqnarray}
    \label{The compact form of TT in intro}
    \calX = [\mX_1,\dots,\mX_N ].
\end{eqnarray}
The storage of the TT format is $O(N\ol d{\ol r^\text{tt}}^2)$ for $\ol d = \max_{i}d_i$ and $\ol r^\text{tt} = \max_{i} r_i^\text{tt}$.
The decomposition of the tensor $\calX$ into the form of \eqref{The element of TT in intro} is generally not unique: not only are the factors $\mX_i(s_i)$ not unique, but also the dimensions of these factors can vary. To introduce the factorization with the smallest possible dimensions $\vr^\text{tt} = (r_1^\text{tt},\ldots,r_{N-1}^\text{tt})$, for convenience, for each $\mX_i$, we put $\{ \mX_i(s_i)\}_{s_i=1}^{d_i}$ together into the following form $L(\mX_i)=\begin{bmatrix}\mX_i^\top(1)  \cdots  \mX_i^\top(d_i) \end{bmatrix}^\top\in\R^{d_i r_{i-1}^\text{tt}\times r_i^\text{tt}}$, where $L(\mX_i)$ is often called the left unfolding of $\mX_i$, if we view $\mX_i$ as a tensor. We say the decomposition \eqref{The element of TT in intro} is minimal if the rank of the left unfolding matrix $L(\mX_i)$ is $r_i^\text{tt}$. The dimensions $\vr^\text{tt} = (r_1^\text{tt},\dots, r_{N-1}^\text{tt})$ of such a minimal decomposition are called the TT ranks of $\calX$. According to \cite{holtz2012manifolds}, there is exactly one set of ranks $\vr^\text{tt}$ for which $\calX$ admits a minimal TT decomposition.
For any TT format $\calX$ of form \eqref{The element of TT in intro}, there always exists a factorization such that $L(\mX_i)$ are orthonormal matrices for all $i \in [N-1]$ (i.e., $L^\top(\mX_i)L(\mX_i) = \mId_{r_i^\text{tt}}$), which is called the left-canonical form or left orthogonal decomposition.

It is worth noting that the tensor train singular value decomposition (TT-SVD), also referred to as the sequential SVD algorithm, offers an efficiently attainable quasi-optimal solution to the left-orthogonal TT decomposition. This technique, introduced in \cite{Oseledets11}, involves employing the SVD \cite{cai2010singular} and reshape operations in an alternating manner.

\paragraph{Special case: Orthogonal CP decomposition} The CP decomposition  requires the least amount of storage by representing the $(s_1,\dots,s_N)$-th element of $\calX$ as
\begin{eqnarray}
    \label{The element in the CP decomposition  in intro}
    \calX(s_1,\dots,s_N)  = \sum_{i=1}^{r^{\text{cp}}} \lambda_i \mB_1(s_1,i)\cdots \mB_N(s_N,i),
\end{eqnarray}
where $\{\lambda_i\}_{i=1}^{r^{\text{cp}}}$ are scalars and $\mB_i\in\R^{d_i\times r^\text{cp}}, i\in[N]$ are factor matrices. Hence, the storage of the CP format is $O(N\ol d r^\text{cp})$ for $\ol d = \max_{i}d_i$. Note that as stated in \cite{cichocki2015tensor}, we can introduce an $N$-order diagonal core $\calD\in\R^{r^\text{cp}\times\cdots\times r^\text{cp}}$ to further define the CP decomposition. Here, $\calD(i,\dots,i) = \lambda_i$ represents the $i$-th diagonal element, while the rest of the elements remain zero.
By comparing \eqref{Thr element of Tucker tensor in intro} with \eqref{The element in the CP decomposition  in intro}, we can observe that if matrices $\mB_i$ in \eqref{The element in the CP decomposition  in intro} are column-wise orthonormal, the CP decomposition will be a special case of the Tucker decomposition, with a diagonal core tensor $\calD$.
While the CP decomposition does not inherently possess an orthonormal structure, the orthogonal CP decomposition \cite{sorensen2012canonical,cheng2016probabilistic,tang2025revisit} has demonstrated enhanced accuracy in recovering latent factors by imposing a column-wise orthonormal constraint on the factors $\mB_i$.  In this paper, we opt to omit the discussion on the orthogonal CP decomposition, as the convergence analysis of the Tucker format can be applied to it without the need for additional exploration.

\paragraph{Factorization approach}
To unify all the tensor decompositions, we use $\calX = [\mZ_1,\dots,\mZ_{\wt N}]\in\R^{d_1\times\cdots \times d_N}$ to denote an $N$-th order tensor parameterized by tensor factors $\mZ_1,\ldots,\mZ_{\wt N}$ with $\wt N\geq N$. For the Tucker decomposition
, we have $\wt N = N+1$, $\mZ_i = \mU_i$, $i=1,\dots, N$ and $\mZ_{\wt N} = \calS$, while for the TT decomposition, we have $\wt N = N$ and $\mZ_i = L(\mX_i)$, $i=1,\dots, N$. To exploit the compact representation, the factorization approach for tensor analysis and recovery problems attempts to optimize the tensor factors by solving
\begin{eqnarray}
    \label{Unified loss function for structured tensor-1}
    \begin{split}
\min_{\mZ_i, i\in[\wt N]} H(\mZ_1,\dots,\mZ_{\wt N})  = h([\mZ_1,\dots,\mZ_{\wt N}]) = h(\calX)
    \end{split}
\end{eqnarray}
with an appropriate loss function $h:\R^{d_1\times \cdots \times d_N}\rightarrow \R$ that depends on the specific problem.

\paragraph{Limitation of existing works} Various iterative algorithms have been proposed for the factorized problem \eqref{Unified loss function for structured tensor-1}, including gradient descent \cite{cai2019nonconvex}, penalized alternating minimization \cite{hao2021sparse}, regularized gradient descent \cite{Han20}, Riemannian gradient descent on the Grassmannian manifold \cite{XiaTC19}, scaled gradient descent \cite{TongTensor21}, and Riemannian gradient descent on the Stiefel manifold \cite{qin2024guaranteed}. A detailed introduction to these algorithms can be found in the related works section. However, existing investigations encounter two fundamental challenges: $(i)$ The majority of existing algorithms for CP or Tucker recovery are tailored for $3$-order, leaving the impact of the tensor order $N$ on initialization requirements and convergence rates largely unexamined. While \cite{Han20} and \cite{qin2024guaranteed} extend to high-order tensors, their theoretical guarantees exhibit notable limitations. Specifically, in \cite{Han20}, both the convergence rate and initialization conditions are susceptible to exponential deterioration with respect to $N$ due to the reliance on an approximately orthonormal structure. In contrast, \cite{qin2024guaranteed} establishes that for tensor train sensing, the initialization requirement and convergence rate of Riemannian gradient descent scale polynomially with $N$. However, this approach employs the $(N+3)\ol r^\text{tt}$-Restricted Isometry Property (RIP), which is significantly more stringent than empirical observations suggest. Moreover, its generalizability to broader tensor recovery problems remains unclear. $(ii)$ The analytical frameworks employed across different algorithms vary considerably, impeding a systematic and unified understanding of tensor recovery. This lack of cohesion complicates theoretical analysis and prevents a direct comparison of different approaches. Therefore, developing a unified analytical framework is imperative for advancing the theoretical foundations of tensor recovery.

\subsection{Our contribution}

In this paper, we aim to develop a unified algorithm and convergence analysis for the factorized problem \eqref {Unified loss function for structured tensor-1} that can be applied to different tensor decompositions. Inspired by the canonical form in the Tucker decomposition and TT decomposition, we assume that the first $\wt N-1$ factors are orthonormal, i.e.,  $\mZ_i^\top\mZ_i = \mId,  i\in[{\wt N}-1]$. Imposing an orthonormal structure on the tensor factors confers several advantages. Firstly, it ensures that the rank of the Tucker format \cite{cichocki2015tensor} and TT format \cite{holtz2012manifolds} remains minimal. Additionally, this constraint improves identifiability \cite{cai2022provable} and mitigates scaling ambiguity \cite{chen2021deep}. For instance, such form is unique up to the insertion of orthogonal matrices between factors. Ultimately, the orthonormal structure leads to more relaxed uniqueness properties \cite{anandkumar2014tensor}. Thus, we solve the following constrained problem
\begin{eqnarray}
    \label{Unified loss function for structured tensor}
    \begin{split}
\min_{\mZ_i, i\in[\wt N]} H(\mZ_1,\dots,\mZ_{\wt N}) & = h([\mZ_1,\dots,\mZ_{\wt N}]),\\
\st \ \mZ_i^\top\mZ_i&=\mId, \  i \in[{\wt N}-1].
    \end{split}
\end{eqnarray}
Since we can view the orthonormal structure as a point on the Stiefel manifold in the Riemannian space,  we solve \eqref{Unified loss function for structured tensor} by a (hybrid) Riemannian gradient descent (RGD) \cite{qin2024guaranteed}:
\begin{equation}
\begin{split}
    & \mZ_i^{(t+1)} = \text{Retr}_{\mZ_i}\Big({{\mZ}_i^{(t)}}-\frac{\mu}{\gamma}\calP_{\text{T}_{\mZ_i} \text{St}}\big(\nabla_{\mZ_i}H(\mZ_1^{(t)},\dots,\mZ_{\wt N}^{(t)}) \big) \Big), \ i \leq {\wt N}-1,\\
    &\mZ_{\wt N}^{(t+1)} ={\mZ_{\wt N}^{(t)}}-\mu\nabla_{\mZ_{\wt N}}H(\mZ_1^{(t)},\dots,\mZ_{\wt N}^{(t)}),
\end{split}
\label{RGD for unified intro}
\end{equation}
where $\mu>0$ is the learning rate for $\{\mZ_i\}_{i=1}^{\wt N}$ and $\gamma>0$ controls the ratio between the learning rates for $\mZ_{\wt N}$ and $\{\mZ_i\}_{i \leq {\wt N}-1}$.
The discrepant learning rates in \eqref{RGD for unified intro} are used to accelerate the convergence rate of $\mZ_{\wt N}$ since the energy of $\mZ_{\wt N}$ and $\{\mZ_i\}_{i \leq {\wt N}-1}$ could be unbalanced. We refer to \Cref{Stiefel and RGD} for definitions of the polar decomposition-based
retraction $\text{Retr}_{\mZ_i}(\cdot)$ and projection $\calP_{\text{T}_{\mZ_i} \text{St}}(\cdot)$ onto the tangent space to the Stiefel manifold.

To analyze the convergence behavior of RGD, we rely on the so-called Riemannian regularity condition  for the factorized objective $H(\mZ_1,\dots,\mZ_{\wt N})$, which will be introduced in detail in \Cref{Riemannian regularity condition for general loss}.
Given the Riemannian regularity condition,  we provide a unified convergence analysis for \eqref{RGD for unified intro} and establish a linear convergence  of the RGD algorithm for solving the factorized problem with a suitable initialization. We summarize this result in the informal theorem below.
\begin{theorem}(informal version)
\label{convergence result of unified tensor}
Assume that the loss function  $H$ in \eqref{Unified loss function for structured tensor} satisfies the Riemannian regularity condition with {a target solution $\calX^\star = [\mZ_1^\star,\dots, \mZ_{\wt N}^\star]$ with ${\mZ_i^\star}^\top\mZ_i^\star = \mId,  i\in[{\wt N}-1]$}. If appropriately initialized, the RGD algorithm in \eqref{RGD for unified intro} converges to $\{\mZ_i^\star\}$ at a linear rate.
\end{theorem}

To establish the Riemannian regularity condition for general cases, we introduce some assumptions on the general loss function $h$.  Drawing inspiration from prior work \cite{Han20}, we will consider a category of loss functions that satisfies the restricted correlated gradient (RCG) condition which will be introduced in \Cref{Riemannian regularity condition for general loss}. Under the RCG condition, we prove that the Riemannian regularity condition holds for orthogonal CP, Tucker, and TT decompositions. 
Furthermore, as will be discussed in detail in \Cref{Riemannian regularity condition for general loss}, the results on the initialization requirement and the  convergence speed  only depend polynomially with respect to the tensor order $N$.
In the context of tensor sensing, under the Restricted Isometry Property (RIP) condition, the RCG condition can be established with the number of measurements scaling linearly with $N$, thereby improving upon the previously required quadratic scaling in $N$ \cite{qin2024guaranteed}. Moreover, the theoretical guarantees can be extended to both regular tensor factorization and tensor completion. Finally, our experimental results substantiate the effectiveness of our theoretical findings.

\subsection{Related work}
Over the past decade, numerous iterative algorithms have been proposed for optimizing matrix/tensor recovery \cite{jain2013low,Tu16,Wang17,Zhu18TSP,Li20,tong2021accelerating,Ma21TSP,TongTensor21,dong2023fast,wang2016tensor,yuan2019high,Rauhut17, rauhut2015tensor,chen2019non,kressner2014low,budzinskiy2021tensor,wang2019tensor,cai2022provable,luo2021low,Han20,XiaTC19,qin2024guaranteed}. One common approach is the projected gradient descent (PGD) algorithm \cite{Rauhut17, rauhut2015tensor, chen2019non, kressner2014low, budzinskiy2021tensor, wang2019tensor,luo2022tensor,qin2024computational} which optimizes the entire tensor in each iteration and applies a projection method to map the iterates to the appropriate tensor decomposition space.  There are two typical algorithms: iterative hard thresholding (IHT) \cite{Rauhut17, rauhut2015tensor, chen2019non,luo2022tensor,qin2024computational}, and the Riemannian gradient descent algorithm applied to the fixed-rank manifold \cite{kressner2014low, budzinskiy2021tensor, wang2019tensor}. In the latter method, the low-rank Tucker/TT decomposition can be viewed as a point in the Riemannian space, specifically the fixed-rank manifold. Both IHT and the Riemannian approach offer local convergence analyses of PGD with HOSVD/TT-SVD for Tucker/TT-based recovery problems. However, PGD requires exceptionally high storage memory, as the entire tensor must be stored during each iteration.

Factorization approaches address this issue by directly estimating the factors within the tensor decomposition, often through gradient-based algorithms \cite{cai2019nonconvex,hao2021sparse,Han20, XiaTC19,TongTensor21, qin2024guaranteed}. The works \cite{cai2019nonconvex,hao2021sparse} analyze the convergence rate of gradient descent and penalized alternating minimization for estimating factors in $3$-order CP decomposition. However, the relationship between convergence rate, initialization requirements, and tensor order $N$ remains unclear. In the Tucker-based recovery problem, \cite{Han20} introduced the regularized gradient descent algorithm by incorporating regularization into the loss function, which fosters an approximately orthonormal structure. However, the theoretical convergence rate and initial conditions could degrade exponentially in relation to $N$, as the estimated factors are not strictly orthonormal but only approximately weighted orthogonal matrices. To achieve a precise orthonormal structure in Tucker recovery, the work \cite{XiaTC19} treated orthonormal factors as points residing on the Grassmannian manifold within the Riemannian space. This concept led to the development of the Grassmannian gradient descent algorithm. However, a comprehensive convergence analysis for this technique remains absent. A scaled gradient descent method is proposed \cite{TongTensor21} to improve dependence on the condition number of the ground-truth tensor. However, its scalability concerning the tensor order $N$ has yet to be established.

The very recent work \cite{qin2024guaranteed} proposed a similar RGD algorithm for the TT sensing problem. However, the analysis is specific to TT decomposition and requires the $(N+3)\ol r^\text{tt}$-RIP condition. In contrast, our work provides a unified approach for analyzing the convergence rate of the factorization method for structured tensors, including orthogonal CP, Tucker, TT, and other tensor networks that can be represented in canonical form. Moreover, it accommodates a class of loss functions that satisfy the RCG condition. When applied to the TT sensing problem, our improved analysis only requires the $4\ol r^\text{tt}$-RIP. A summary comparing our method with previous works is provided in \Cref{Comparison among different algorithms}.

\begin{table*}[!ht]
\renewcommand{\arraystretch}{1.8}
\begin{center}
\caption{Comparison of existing factorization approaches with the proposed method, where GD denotes gradient descent, AM represents alternating minimization and RCG   indicates the restricted correlated gradient property defined in {Definition}~\ref{The defi in the RCG conditoin}. }
\label{Comparison among different algorithms}
{\begin{tabular}{|c||c|c|c|c|c|} \hline  {Algorithm} &{Decomposition}  &{Order} & {Loss} & {\makecell{Convergence \\ in terms of $N$}} & {\makecell{Initialization \\ in terms of $N$}}
\\\hline {GD \cite{cai2019nonconvex}} &  CP & $3$ & $l_2$ & unknown & unknown
\\\hline {Penalized AM \cite{hao2021sparse}} &  CP &  $3$ &  $l_2$ & unknown & unknown
\\\hline
{Regularized GD \cite{Han20}} &  Tucker &  $N$ &  RCG & exponential & exponential
\\\hline
{Grassmannian GD \cite{XiaTC19}} & Tucker &  $3$ &  $l_2$ & unknown & unknown
\\\hline
{Scaled GD \cite{TongTensor21}} & Tucker &  $3$ &  $l_2$ & unknown & unknown
\\\hline
{RGD \cite{qin2024guaranteed}} & TT &  $N$ &  $l_2$ & polynomial & polynomial
\\\hline
{Our method} & \makecell{orthogonal CP, \\ Tucker, TT} &  $N$ &  RCG & polynomial & polynomial
\\\hline
\end{tabular}}
\end{center}
\end{table*}

{\bf Notation}: We use calligraphic letters (e.g., $\calY$) to denote tensors,  bold capital letters (e.g., $\mY$) to denote matrices, except for $\mX_i$ which denotes the $i$-th $3$-order tensor factors in the TT format, bold lowercase letters (e.g., $\vy$) to denote vectors, and italic letters (e.g., $y$) to denote scalar quantities.  Elements of matrices and tensors are denoted in parentheses, as in Matlab notation. For example, $\calY(s_1, s_2, s_3)$ denotes the element in position $(s_1, s_2, s_3)$ of the $3$-order tensor $\calY$.
The inner product of $\calA\in\R^{d_1\times\dots\times d_N}$ and $\calB\in\R^{d_1\times\dots\times d_N}$ can be denoted as $\<\calA, \calB \> = \sum_{s_1=1}^{d_1}\cdots \sum_{s_N=1}^{d_N} \calA(s_1,\dots,s_N)\calB(s_1,\dots,s_N) $.
The vectorization of  $\calX\in\R^{d_1\times\dots\times d_N}$, denoted as $\text{vec}(\calX)$, transforms the tensor $\calX$ into a vector. The $(s_1, \dots, s_N)$-th element of $\calX$ can be found in the vector $\text{vec}(\calX)$ at the position $s_1 + d_1(s_2-1) + \cdots + d_1d_2 \cdots d_{N-1}(s_N-1)$. We define the $(s_1, \dots, s_N)$-th element of the mode-$i$ tensor-matrix product $\calS\times_i \mU_i\in \R^{r_1^\text{tk}\times \cdots\times r_{i-1}^\text{tk}\times d_i\times r_{i+1}^\text{tk}\times \cdots r_N^\text{tk}}$ as $(\calS\times_i \mU_i)({s_1,\dots, s_N})=\sum_{j=1}^{r_i^\text{tk}}\calS({s_1,\dots, s_{i-1},j,s_{i+1},\dots, s_N})\mU_i({s_i, j})$. Furthermore, we introduce an alternative notation for the Tucker decomposition: $[\![\calS; \mU_1, \mU_2,\dots, \mU_N ]\!] = \calS\times_1\mU_1\times_2\cdots\times_N \mU_N $.
$\|\calX\|_F = \sqrt{\<\calX, \calX \>}$ is the Frobenius norm of $\calX$.
$\|\mX\|$ and $\|\mX\|_F$ respectively represent the spectral norm and Frobenius norm of $\mX$.
$\sigma_{i}(\mX)$ is the $i$-th singular value of $\mX$.
$\|\vx\|_2$ denotes the $l_2$ norm of $\vx$.
$\times_i$ is the  mode-$i$ tensor-matrix product.
$\otimes$ is the kronecker product.
$\ol \otimes$ denotes the kronecker product between submatrices in two block matrices. Its detailed definition and property are shown in {Appendix} \ref{Technical tools used in proofs}.
For a positive integer $K$, $[K]$ denotes the set $\{1,\dots, K \}$.
For two positive quantities $a,b\in \real$,  $b = O(a)$ means $b\leq c a$ for some universal constant $c$; likewise, $b = \Omega(a)$ represents $b\ge ca$ for some universal constant $c$.

Next, we define singular values and condition numbers for the Tucker and TT formats, respectively. (1) Based on the matricization operator $\calM_i(\calS\times_1\mU_1\times_2\cdots\times_N \mU_N)=\mU_i\calM_i(\calS)(\mU_{i-1}\otimes\cdots\otimes\mU_1\otimes\mU_N\otimes\cdots\otimes\mU_{i+1})^\top\in\R^{d_i\times \Pi_{j\neq i}d_j}$, we can define the largest and smallest singular value of $\calX$ in the Tucker format as $\overline{\sigma}_{\text{tk}}(\calX)=\max_{i=1}^{N}\|\calM_i(\calX)\|$ and $\underline{\sigma}_{\text{tk}}(\calX)=\min_{i=1}^{N} \sigma_{r_i^\text{tk}}(\calM_i(\calX))$. Additionally, we can define the condition number of $\calX$ in the Tucker format as $\kappa_\text{tk}(\calX)=\frac{\overline{\sigma}_{\text{tk}}(\calX)}{\underline{\sigma}_{\text{tk}}(\calX)}$.
(2) We note that $r_i^\text{tt}$ also equals the rank of the $i$-th unfolding matrix $\calX^{\<i\>}\in\R^{(d_1\cdots d_k)\times (d_{k+1}\cdots d_N)}$ of the tensor $\calX$, where the $(s_1\cdots s_i, s_{i+1}\cdots s_N)$-th element\footnote{ Specifically, $s_1\cdots s_i$ and $s_{i+1}\cdots s_N$ respectively represent the $(s_1+d_1(s_2-1)+\cdots+d_1\cdots d_{i-1}(s_i-1))$-th row and $(s_{i+1}+d_{i+1}(s_{i+2}-1)+\cdots+d_{i+1}\cdots d_{N-1}(s_N-1))$-th column.
} of $\calX^{\<i\>}$ is given by $\calX^{\<i\>}(s_1\cdots s_i,  s_{i+1}\cdots s_N) = \calX(s_1,\dots, s_N)$.
With the $i$-th unfolding matrix $\calX^{\<i\>}$\footnote{We can also define the $i$-th unfolding matrix as $\calX^{\< i \>} = \mX^{\leq i}\mX^{\geq i+1}$, where each row of the left part $\mX^{\leq i}$ and each column of the right part $\mX^{\geq i+1}$ can be represented as $\mX^{\leq i}(s_1\cdots s_i,:) = \mX_1(s_1)\cdots \mX_i(s_i)$ and $\mX^{\geq i+1}(:,s_{i+1}\cdots s_{N}) = \mX_{i+1}(s_{i+1})\cdots \mX_N(s_{N})$. When factors are in left-orthogonal form, we have ${\mX^{\leq i}}^\top\mX^{\leq i} = \mId_{r_i^\text{tt}}$ and $\sigma_j(\calX^{\< i\>}) = \sigma_j( \mX^{\geq i+1} ), j\in[N-1]$.} and TT ranks, we can further define its smallest singular value $\underline{\sigma}_\text{tt}(\calX)=\min_{i=1}^{N-1}\sigma_{r_i^\text{tt}}(\calX^{\<i\>})$, its largest singular value $\overline{\sigma}_\text{tt}(\calX)=\max_{i=1}^{N-1}\sigma_{1}(\calX^{\<i\>})$ and condition number $\kappa_\text{tt}(\calX)=\frac{\overline{\sigma}_\text{tt}(\calX)}{\underline{\sigma}_\text{tt}(\calX)}$.
In the end, we define $\ol d = \max_{i}d_i$, $\ol r^\text{tk} = \max_{i}r_i^\text{tk}$, and $\ol r^\text{tt} = \max_{i} r_i^\text{tt}$.
Without specifying a particular tensor decomposition, we use the notations ${\overline{\sigma}(\calX)}$ and $\ol r$ to simultaneously represent the cases of $\overline{\sigma}_\text{tk}(\calX)$ with $\ol r^\text{tk}$ and $\overline{\sigma}_\text{tt}(\calX)$ with $\ol r^\text{tt}$.

\section{Local Convergence of Riemannian Gradient Descent for General Loss Functions}
\label{LocalConvergenceof General loss}

In this section, we will study the convergence of Riemannian gradient descent (RGD) \eqref{RGD for unified intro} for a general loss function $h$ that satisfies the restricted correlated gradient (RCG) condition.  

\subsection{Riemannian gradient descent}
\label{Stiefel and RGD}

We begin by introducing some useful results about the Stiefel manifold within the Riemannian space. The Stiefel manifold, denoted as $\text{St}(m,n)=\{\mG\in\R^{m \times n}: \mG^\top\mG = \mId_{n}\}$, consists of all $m \times n$ column-wise orthonormal matrices. We can regard $\text{St}(m,n)$ as an embedded submanifold of a Euclidean space and further define $\text{T}_{\mG} \text{St}:=\{\mA\in\R^{m\times n}: \mA^\top\mG+\mG^\top\mA={\bm 0} \}$  as the tangent space to the Stiefel manifold $\text{St}(m,n)$ at the point $\mG \in \text{St}(m,n)$. Given an arbitrary matrix $\mB\in\R^{m\times n}$, the projection of $\mB$ onto $\text{T}_{\mG} \text{St}$ can be expressed as follows \cite{absil2008optimization}:
\begin{eqnarray}
    \label{Projection to the Tangent space on the Stiefel manifold unified}
    \calP_{\text{T}_{\mG} \text{St}}(\mB)=\mB-\frac{1}{2}\mG(\mB^\top\mG+\mG^\top\mB),
\end{eqnarray}
and its orthogonal complement is
\begin{eqnarray}
    \label{orthogonal complement to the Tangent space on the Stiefel manifold unified}
    \calP_{\text{T}_{\mG} \text{St}}^{\perp}(\mB)=\frac{1}{2}\mG(\mB^\top\mG+\mG^\top\mB).
\end{eqnarray}
Note that when we have a gradient $\mB$ defined in the Hilbert space, we can use the projection operator \eqref{Projection to the Tangent space on the Stiefel manifold unified} to compute the Riemannian gradient $\calP_{\text{T}_{\mG} \text{St}}(\mB)$ on the tangent space of the Stiefel manifold. For an update $\wh \mG = \mG - c \calP_{\text{T}_{\mG} \text{St}}(\mB)$ with a positive constant $c$, we can utilize the polar decomposition-based retraction to project it back onto the Stiefel manifold:
\begin{eqnarray}
    \label{polar decomposition-based retraction sec2}
    \text{Retr}_{\mG}(\wh \mG)=\wh \mG(\wh \mG^\top\wh \mG)^{-\frac{1}{2}}.
\end{eqnarray}

Recall the (hybrid) RGD algorithm in \eqref{RGD for unified intro}:
\begin{equation}
\begin{split}
    & \mZ_i^{(t+1)} =\text{Retr}_{\mZ_i}\Big({{\mZ}_i^{(t)}}-\frac{\mu}{\gamma}\calP_{\text{T}_{\mZ_i} \text{St}}\big(\nabla_{\mZ_i}H(\mZ_1^{(t)},\dots,\mZ_{\wt N}^{(t)}) \big) \Big), \ i \leq {\wt N}-1,\\
    &\mZ_{\wt N}^{(t+1)} ={\mZ_{\wt N}^{(t)}}-\mu\nabla_{\mZ_{\wt N}}H(\mZ_1^{(t)},\dots,\mZ_{\wt N}^{(t)}).
\end{split}
\label{RGD for unified}
\end{equation}
For the Tucker decomposition, $\wt N = N+1$, $\mZ_i = \mU_i, i = 1\in [N]$, $\mZ_{\wt N} = \calS$, and we set $\gamma = \overline{\sigma}_{\text{tk}}^2(\calX^\star)$, while for the TT decomposition, $\wt N = N$, $\mZ_i = L(\mX_i), i \in [N]$, and we set $\gamma = \overline{\sigma}_{\text{tt}}^2(\calX^\star)$. Here $\calX^\star = [\mX_1^\star,\ldots,\mX_{\wt N}^\star]$ denotes the tensor formed by a target solution $\{\mX_i^\star\}$ of \eqref{Unified loss function for structured tensor}. The choice of the discrepant learning rate ratio $\gamma$ is to balance the convergence speed of the first $\wt N-1$ factors $\mX_i^\star$ and the last factor $\mX_{\wt N}^\star$: for a Tucker format tensor $\calX^\star = [\![\calS^\star; \mU_1^\star,\dots, \mU_{N}^\star ]\!]$, $\|\mU_i^\star\|^2 = 1$ and $\|\calS^\star\|^2 = \|\calX^\star\|^2 = \overline{\sigma}_{\text{tk}}^2(\calX^\star)$ are always satisfied,  thus we set $\gamma = \overline{\sigma}_{\text{tk}}^2(\calX^\star)$ to balance them; while for a  TT format tensor $\calX^\star = [\mX_1^\star,\dots, \mX_N^\star]$, $\|L(\mX_i^\star)\|^2 = 1$ and $\|\begin{bmatrix}\mX_N^\star(1) &  \cdots &  \mX_N^\star(d_N) \end{bmatrix}\|^2 = \sigma_1^2({\calX^\star}^{\< N-1 \>})\leq \ol{\sigma}^2(\calX^\star)$, and we employ $\gamma = \ol{\sigma}_\text{tt}^2(\calX^\star)$ to balance them.
{To unify the notation, we define $\gamma = \ol{\sigma}^2(\calX^\star)$.} Ultimately, the expressions of gradients for Tucker and TT are respectively shown in \eqref{gradients of Riemannain gradient descent algorithm factorization 1toN general} and \eqref{gradients of Riemannain gradient descent algorithm factorization S general} of {Appendix} \ref{proof of regularity condition of general loss in Tucker format} and \eqref{the gradient of general loss in the TT format} of {Appendix} \ref{proof of regularity condition of general loss in TT format}.

\subsection{Distance measure} Before analyzing the convergence of RGD, we introduce an appropriate metric to quantify the distinctions between factors $\{\mZ_i\}$ and $ \{\mZ_i^\star\}$. Note that there exists rotation ambiguity among $ \{\mZ_i^\star\}$, i.e., $[\mZ_1^\star,\dots,\mZ_{\wt N}^\star] = [(\mZ_1^\star, \mR_1),\dots,(\mZ_{\wt N}^\star, \mR_{\wt N})]$ for orthonormal matrices $\{\mR_i\}_{i=1}^{\wt N}$, where the exact formulation of $(\mZ_i^\star, \wt\mR_i)$ is contingent on the specific tensor decomposition. For instance, for the Tucker format tensor $\calX^\star=[\![\calS^\star; \mU_1^\star, \dots, \mU_{N}^\star ]\!]$, we have rotation ambiguity among factors \cite{zhang2017matrix} as
$
[\![\calS^\star; \mU_1^\star, \dots, \mU_N^\star ]\!] 
=[\![ [\![\calS^\star; \mR_1^\top,\dots, \mR_N^\top  ]\!];   \mU_1^\star\mR_1, \dots, \mU_N^\star\mR_N    ]\!]
$ for any orthonormal matrices $\mR_i\in\O^{r_i^\text{tk}\times r_i^\text{tk}}, i\in[N]$. Here we have $(\mZ_i^\star, \mR_i) = \mU_i^\star\mR_i, i=1,\dots, N$ and $(\mZ_{N+1}^\star, \mR_{N+1}) =  [\![\calS^\star; \mR_1^\top,\dots, \mR_N^\top  ]\!] $.  Similarly, there also exists rotation ambiguity among factors in the TT format $\calX^\star = [{\mX}_1^\star,\dots,{\mX}_N^\star ]$ as $\Pi_{i=1}^N\mX_i^\star(s_i) = \Pi_{i=1}^N\mR_{i-1}^\top\mX_i^\star(s_i)\mR_i$ for any orthonormal matrix $\mR_i\in\O^{r_i^\text{tt}\times r_i^\text{tt}}$ (with $\mR_0 = \mR_N = 1$). Here, we have
$(\mZ_i^\star, \mR_i) = L_{\mR}({\mX}_i^\star): = \begin{bmatrix}\mR_{i-1}^\top\mX_{i}^\star(1)\mR_i\\ \vdots \\ \mR_{i-1}^\top\mX_{i}^\star(d_i)\mR_i\end{bmatrix} , i=1,\dots, N$. Thus, we propose the following measure to capture the distance between two sets of factors:
\begin{eqnarray}
\label{BALANCED NEW DISTANCE BETWEEN TWO Weight matrices unified}
\text{dist}^2(\{\mZ_i\},\{\mZ^\star_i\}) &\!\!\!\!=\!\!\!\!& \min_{\mR_i, i\in[{\wt N}]}\sum_{i=1}^{{\wt N}-1} \gamma \|\mZ_i- (\mZ_i^\star, \mR_i)\|_F^2 + \|\mZ_{\wt N}- (\mZ_{\wt N}^\star, \mR_{\wt N})\|_F^2\nonumber\\
&\!\!\!\!=\!\!\!\!&\sum_{i=1}^{{\wt N}-1} \gamma \|\mZ_i- (\mZ_i^\star, \wt\mR_i)\|_F^2 + \|\mZ_{\wt N}- (\mZ_{\wt N}^\star, \wt\mR_{\wt N})\|_F^2,
\end{eqnarray}
where the weight $\gamma$ (set to be $\overline{\sigma}^2(\calX^\star)$ in \eqref{RGD for unified}) is used to balance the two terms since the energy of
 $\{\mZ_i\}_{i \leq {\wt N}-1}$ and $\mZ_{\wt N}$  could be unbalanced.

The following result establishes a connection between $\text{dist}^2(\{\mZ_i\},\{\mZ^\star_i\})$ and $\|\calX-\calX^\star\|_F^2$. First, we respectively define $\ol\sigma(\calX)$ and $\kappa(\calX)$ as the maximum singular value and condition number of $\calX\in\R^{d_1\times \cdots \times d_N}$.
\begin{lemma}
\label{LOWER BOUND OF TWO DISTANCES main paper general}
For any two tensors $\calX$ and $\calX^\star $, we assume that  $\overline{\sigma}(\calX)\leq \frac{3\overline{\sigma}(\calX^\star)}{2}$. Then $\text{dist}^2(\{\mZ_i\},\{\mZ^\star_i\})$ and $\|\calX-\calX^\star\|_F^2$ satisfy
\begin{eqnarray}
    \label{LOWER BOUND OF TWO DISTANCES_1 main paper general}
  \zeta_1(\wt{N},\kappa(\calX^\star))\text{dist}^2(\{\mZ_i\} \leq \|\calX-\calX^\star\|_F^2    \leq \frac{9\wt{N}}{4}\text{dist}^2(\{\mZ_i\},\{\mZ^\star_i\}).
\end{eqnarray}
\end{lemma}
Specifically, we have $ \zeta_1(\wt{N},\kappa(\calX^\star)) = \begin{cases}
\frac{1}{(9N^2+14N+1)\kappa^2_\text{tk}(\calX^\star)}, &  \text{Tucker} \\
\frac{1}{8(N+1+\sum_{i=2}^{N-1}r_i^\text{tt})\kappa_\text{tt}^2(\calX^\star)}, &  \text{TT}
\end{cases}$. The detailed proof for Tucker decomposition is provided in \Cref{Lower bound of two distance in the tucker tensor_1} of {Appendix} \ref{Technical tools used in proofs}. In addition, the result for TT decomposition directly follows \cite[Lemma 1]{qin2024guaranteed}. This relationship can guarantee linear convergence of $\|\calX-\calX^\star\|_F^2$ when  $\text{dist}^2(\{\mZ_i\},\{\mZ^\star_i\})$ converges linearly.

\subsection{Riemannian regularity condition}
\label{Riemannian regularity condition for general loss}

To study the local convergence of RGD, we rely on following the Riemannian regularity condition  for $H$ in \eqref{Unified loss function for structured tensor}.
\begin{definition} (Riemannian regularity condition)
\label{definition of regularity condition unified}
Given a target tensor $\calX^\star = [\mZ_1^\star,\dots, \mZ_{\wt N}^\star]$,
we say that the loss function $H$ in \eqref{Unified loss function for structured tensor} satisfies a Riemannian regularity condition, denoted by $\text{RRC}(a_1,a_2,a_3)$, if for all  $[\mZ_1,\dots,\mZ_{\wt N}]\in\{ [\mZ_1,\dots,\mZ_{\wt N}]:  \text{dist}^2(\{\mZ_i\},\{\mZ^\star_i\})\leq a_1  \}$,  the following inequality holds:
\begin{align}
    \label{definition of regularity condition for gene loss unified}
    &\hspace{0.4cm} \sum_{i=1}^{{\wt N}-1}\bigg\<\mZ_i- (\mZ_i^\star, \wt\mR_i), \calP_{\text{T}_{\mZ_i} \text{St}}\big(\nabla_{\mZ_i}H(\mZ_1,\dots,\mZ_{\wt N}) \big)\bigg\>+\bigg\<\mZ_{\wt N}- (\mZ_{\wt N}^\star, \wt\mR_{\wt N}),  \nabla_{\mZ_{\wt N}}H(\mZ_1,\dots,\mZ_{\wt N}) \bigg\>\nonumber\\
    &\geq a_2\text{dist}^2(\{\mZ_i\},\{\mZ^\star_i\})
    +a_3\bigg(\frac{1}{\gamma}\sum_{i=1}^{{\wt N}-1}\|\calP_{\text{T}_{\mZ_i} \text{St}}\big(\nabla_{\mZ_i}H(\mZ_1,\dots,\mZ_{\wt N}) \big)\|_F^2+\|\nabla_{\mZ_{\wt N}}H(\mZ_1,\dots,\mZ_{\wt N})\|_F^2  \bigg)
\end{align}
for the set of orthonormal matrices $\wt\mR_i, i\in[{\wt N}]$.
\end{definition}
The Riemannian regularity condition \eqref{definition of regularity condition for gene loss unified} extends previous regularity conditions used for studying nonconvex problems~\cite{chen2015solving,candes2015phase,zhu2019linearly,chi2019nonconvex,yonel2020deterministic} to our settings on the Stiefel manifold.
Roughly speaking, the condition in \eqref{definition of regularity condition for gene loss unified} characterizes a form of strong convexity along specific trajectories defined by the tensor decomposition factors, thereby ensuring a positive correlation between the errors $\{ \mZ_i- (\mZ_i^\star, \wt\mR_i) \}_{i=1}^{\wt N}$ and  negative Riemannian search directions $\{ \calP_{\text{T}_{\mZ_i} \text{St}}\big(\nabla_{\mZ_i}H(\mZ_1,\dots,\mZ_{\wt N}) \big) \}_{i=1}^{\wt N}$ with $\calP_{\text{T}_{\mZ_{\wt N}}} \text{St} = \calI$ in
the Riemannian space, i.e., negative Riemannian directions pointing towards the true factors. To establish the above Riemannian regularity condition for general cases, we assume that the loss function $h$ satisfies the Restricted Correlated Gradient (RCG) condition \cite{Han20}.
\begin{definition} (Restricted Correlated Gradient)
\label{The defi in the RCG conditoin}
Let $h$ be a real-valued function. We say $h$ satisfies $\text{RCG}(\alpha,\beta,\calC)$ condition for $\alpha, \beta>0$ and the set $\mathcal{C}$ if
\begin{eqnarray}
    \label{RCG condition}
    \<\nabla h(\calX)-\nabla h(\calX^\star), \calX - \calX^\star\>\geq\alpha\|\calX - \calX^\star\|_F^2+\beta\|\nabla h(\calX)-\nabla h(\calX^\star)\|_F^2
\end{eqnarray}
for any $\calX\in\calC$. Here $\calX^\star$ is some fixed target parameter.
\end{definition}
The RCG condition is a generalization of the strong convexity condition. When $h$ represents the mean squared error (MSE) loss, i.e., $h(\calX) = \|\calX - \calX^\star\|_F^2$ which is commonly used in the convergence analysis of tensor recovery \cite{TongTensor21,cai2022provable,luo2021low,Han20,XiaTC19}, it satisfies the RCG condition with $\alpha = \beta =1$.
Additionally, the RCG condition can also accommodate other typical functions, such as the negative log-likelihood function \cite{Han20}, particularly when the noise follows a Gaussian distribution, and other contexts such as in the tensor sensing problems which will be discussed in the next section. Also, note that the RCG condition only requires \eqref{RCG condition} to hold within a restricted set $\calC$ that includes the target tensor, such as those represented by Tucker  (including orthogonal CP) and TT decompositions.
RCG is sufficient for establishing the Riemannian regularity condition for Tucker and TT decompositions.

\begin{lemma}
Assume that $h$ satisfies $\text{RCG}(\alpha,\beta,\calC)$ condition where $\calC$ contains the subspaces associated with the Tucker and TT decompositions. When $ [\mZ_1,\dots,\mZ_{\wt N}]$ satisfies $\text{dist}^2(\{\mZ_i\},\{\mZ^\star_i\})\leq a_1$ for all $i$, in both the Tucker and TT cases (as summarized in \Cref{summary of all parameters in Riemannian regularity condition}), the constants $a_2$ and $a_3$ in the Riemannian regularity condition follow accordingly.
\begin{table}[!ht]
\renewcommand{\arraystretch}{2}
\begin{center}
\caption{Parameters in the Riemannian regularity conditions for Tucker and TT decompositions}
\label{summary of all parameters in Riemannian regularity condition}
{\begin{tabular}{|c||c|c|}
\hline
Tensor format & Tucker                                                                                  & TT                                                                                                                   \\ \hline
$a_1$         & $\frac{8\alpha\beta\underline{\sigma}_{\text{tk}}^2(\calX^\star)}{9N(N+2)(9N^2+14N+1)}$ & $\frac{\alpha\beta\underline{\sigma}^2_\text{tt}(\calX^\star)}{9(N^2-1)(N+1+\sum_{i=2}^{N-1}r_i^\text{tt})}$         \\ \hline
$a_2$         & $\frac{\alpha}{2(9N^2+14N+1)\kappa^2_\text{tk}(\calX^\star)}$                           & $\frac{\alpha}{16(N+1+\sum_{i=2}^{N-1}r_i^\text{tt})\kappa_\text{tt}^2(\calX^\star)}$ \\ \hline
$a_3$         & $\frac{\beta}{9N+4}$                                                                    & $\frac{\beta}{9N-5}$                                                                                                 \\ \hline
\end{tabular}}{}
\end{center}
\end{table}
\end{lemma}
The proof is given in {Appendix} \ref{proof of regularity condition of general loss in Tucker format} and {Appendix} \ref{proof of regularity condition of general loss in TT format}. {It is worth emphasizing that the initialization requirement $a_1$ exhibits only polynomial dependence on $N$, ensuring the scalability of the proposed method. Furthermore, $a_2$, which captures a trajectory-dependent notion of strong convexity, and $a_3$, which governs the gradient norm contribution, both decay at most polynomially with respect to  $N$. This guarantees that the RGD algorithm is free from exponentially deteriorating convergence rates and gradient vanishing phenomena as the problem size grows.}
Note that the above result can be extended to other tensor decompositions, where the constants $a_1$, $a_2$, and $a_3$ depend on the specific decomposition format.

\subsection{Convergence analysis}
\label{Convergence analysis for general loss}

With the $\text{RRC}(a_1,a_2,a_3)$, we now provide a unified convergence analysis for the RGD \eqref{RGD for unified}. Supposing that $\{\mZ_i^{(t)}\}$ satisfies $\text{dist}^2(\{\mZ_i^{(t+1)}\},\{\mZ^\star_i\}) \le a_1$, then
\begin{eqnarray}
\label{expansion dist t+1 unified}
&\!\!\!\!\!\!\!\!& \text{dist}^2(\{\mZ_i^{(t+1)}\},\{\mZ^\star_i\}) = \sum_{i=1}^{{\wt N}-1} \gamma \|\mZ_i^{(t+1)}- (\mZ_i^\star, \wt\mR_i^{(t+1)})\|_F^2 + \|\mZ_{\wt N}^{(t+1)}- (\mZ_{\wt N}^\star, \wt\mR_{\wt N}^{(t+1)})\|_F^2\nonumber\\
&\!\!\!\!\leq \!\!\!\!&\sum_{i=1}^{{\wt N}-1} \gamma \|\mZ_i^{(t+1)}- (\mZ_i^\star, \wt\mR_i^{(t)})\|_F^2 + \|\mZ_{\wt N}^{(t+1)}- (\mZ_{\wt N}^\star, \wt\mR_{\wt N}^{(t)})\|_F^2\nonumber\\
&\!\!\!\!\leq \!\!\!\!& \sum_{i=1}^{{\wt N}-1} \gamma \|\mZ_i^{(t)}-\frac{\mu}{\gamma}\calP_{\text{T}_{\mZ_i} \text{St}}\big(\nabla_{\mZ_i}H(\mZ_1^{(t)},\dots,\mZ_{\wt N}^{(t)}) \big)- (\mZ_i^\star, \wt\mR_i^{(t)})\|_F^2\nonumber\\
&\!\!\!\!\!\!\!\!& + \|\mZ_{\wt N}^{(t)}-\mu\nabla_{\mZ_{\wt N}}H(\mZ_1^{(t)},\dots,\mZ_{\wt N}^{(t)})- (\mZ_{\wt N}^\star, \wt\mR_{\wt N}^{(t)})\|_F^2\nonumber\\
&\!\!\!\!=\!\!\!\!& \text{dist}^2(\{\mZ_i^{(t)}\},\{\mZ^\star_i\}) - 2\mu \bigg(\sum_{i=1}^{{\wt N}-1}\bigg\<\mZ_i^{(t)}- (\mZ_i^\star, \wt\mR_i^{(t)}), \calP_{\text{T}_{\mZ_i} \text{St}}\big(\nabla_{\mZ_i}H(\mZ_1^{(t)},\dots,\mZ_{\wt N}^{(t)}) \big)\bigg\>\nonumber\\
&\!\!\!\!\!\!\!\!&+\bigg\<\mZ_{\wt N}^{(t)}- (\mZ_{\wt N}^\star, \wt\mR_{\wt N}^{(t)}),  \nabla_{\mZ_{\wt N}}H(\mZ_1^{(t)},\dots,\mZ_{\wt N}^{(t)}) \bigg\>  \bigg) + \mu^2\bigg(\frac{1}{\gamma}\sum_{i=1}^{{\wt N}-1}\|\calP_{\text{T}_{\mZ_i} \text{St}}\big(\nabla_{\mZ_i}H(\mZ_1^{(t)},\dots,\mZ_{\wt N}^{(t)}) \big)\|_F^2\nonumber\\
&\!\!\!\!\!\!\!\!&+\|\nabla_{\mZ_{\wt N}}H(\mZ_1^{(t)},\dots,\mZ_{\wt N}^{(t)})\|_F^2  \bigg)\leq (1-2a_2\mu)\text{dist}^2(\{\mZ_i^{(t)}\},\{\mZ^\star_i\}),
\end{eqnarray}
where the second inequality follows from the nonexpansiveness property described in \Cref{NONEXPANSIVENESS PROPERTY OF POLAR RETRACTION_1} of {Appendix} \ref{Technical tools used in proofs}, and the last line follows from \eqref{definition of regularity condition for gene loss unified} in \Cref{definition of regularity condition unified} and the assumption that step size $\mu\leq 2a_3$. Based on the discussion above, we obtain the following convergence guarantee.
\begin{theorem}
\label{Local Convergence of general_Theorem}
Suppose that  $H$ satisfies the Riemannian regularity condition. With  initialization $\{\mZ_i^{(0)}\}$ satisfying $\text{dist}^2(\{\mZ_i^{(0)}\},\{\mZ^\star_i\})\leq a_1$ and step size $\mu\leq2a_3$,  the iterates $\{ {\mZ_i}^{(t)} \}_{t\geq 0}$ generated by RGD will converge linearly to $\mZ_i^\star$ (up to rotation):
\begin{eqnarray}
    \label{Local Convergence of Riemannian in the TT general_Theorem_11}
    \text{dist}^2(\{\mZ_i^{(t+1)}\},\{\mZ^\star_i\})\leq(1-2a_2\mu)\text{dist}^2(\{\mZ_i^{(t)}\},\{\mZ^\star_i\}),
\end{eqnarray}
where $a_i,i=1,2,3$ are the constants in the $\text{RRC}(a_1,a_2,a_3)$ in \Cref{definition of regularity condition unified}.
\end{theorem}
According to Table \ref{summary of all parameters in Riemannian regularity condition} of the parameters $a_i$ for the Tucker and TT decompostions,
\Cref{Local Convergence of general_Theorem} ensures that both the initial requirement and convergence rate depends only polynomially, rather than exponentially, as $N$ increases, for both Tucker and TT. It is worth noting that, beyond the Tucker (including orthogonal CP) and TT decompositions, RGD and its corresponding analysis can also be extended to the Hierarchical Tucker and tensor network decompositions. The Hierarchical Tucker \cite{grasedyck2010hierarchical, da2015optimization} and tensor network \cite{cichocki2016tensor, orus2019tensor} representations can be seen as extensions of the Tucker and TT models, respectively. However, for an $N$-order Hierarchical Tucker or tensor network decomposition, multiple equivalent representations exist, making it challenging to formulate a unified framework that encompasses all scenarios. For a specific form of Hierarchical Tucker or tensor network decomposition where most factors are orthonormal except for one, the initialization requirement and convergence rate of RGD remain polynomially dependent on the tensor order $N$, following the same analytical approach discussed earlier. Therefore, in this paper, we primarily focus on widely used tensor decompositions, including Tucker (encompassing orthogonal CP) and TT decompositions.

\section{Applications in Tensor Sensing}
\label{Typical Applications in tensor}

In this section, we will study the application of RGD to the tensor sensing problem, where the goal is to recover a low-rank tensor $\calX^\star\in\R^{d_1\times \cdots \times d_{N}}$ from limited measurements $\vy\in\R^m$ of the form
\begin{eqnarray}
    \label{Definition of tensor sensing}
    \vy = \calA(\calX^\star) =\begin{bmatrix}
          y_1 \\
          \vdots \\
          y_m
        \end{bmatrix} = \begin{bmatrix}
          \<\calA_1,\calX^\star\> \\
          \vdots \\
          \<\calA_m,\calX^\star\>
        \end{bmatrix}\in\R^m.
\end{eqnarray}
Here, $\calA(\calX^\star): \R^{d_{1}\times  \cdots \times d_{N}}\rightarrow \R^m$ represents a linear map that characterizes the measurement process, and the number of measurements $m$ is much smaller than the total number of entries in $\calX^\star$, i.e., $m\ll d_1\cdots d_N$.
This problem has appeared in many applications such as quantum state tomography \cite{lidiak2022quantum,qin2024quantum}, 3D imaging \cite{guo2011tensor}, high-order interaction pursuit \cite{hao2020sparse}, neuroimaging analysis \cite{zhou2013tensor,li2017parsimonious}, and more.

We will consider the following common loss that captures the  discrepancy between the measurements $\vy$ and the linear mapping of the estimated tensor $\calX$ as
\begin{eqnarray}
    \label{General loss for G}
     G(\calX) = \frac{1}{2m}\|\calA(\calX) - \vy\|_2^2.
\end{eqnarray}
Our goal is to show that this loss function $G$ satisfies the RCG condition for Tucker and TT format tensors. Towards that goal, we first present the Restricted Isometry Property (RIP) \cite{donoho2006compressed,candes2006robust,
candes2008introduction,recht2010guaranteed,grotheer2021iterative,Rauhut17,qin2024quantum}, which ensures the proximity of the energy $\|\calA(\calX)\|_2^2$ to $\|\calX\|_F^2$, making it a pivotal factor for our analytical pursuits.
\begin{definition} (Restricted isometry property (RIP))
\label{RIP condition fro the Tucker sensing Lemma}
A linear operator $\calA: \R^{d_{1}\times  \cdots \times d_{N}}\rightarrow \R^m$ is said to satisfy the $\ol r$-RIP if
\begin{eqnarray}
    \label{RIP condition fro the Tucker TT sensing}
    (1-\delta_{\ol r})\|\calX\|_F^2\leq \frac{1}{m}\|\mathcal{A}(\calX)\|_2^2\leq(1+\delta_{\ol r})\|\calX\|_F^2,
\end{eqnarray}
holds for any tensor decomposition $\calX$ with rank $\ol r$. Here $\delta_{\ol r} \in(0,1)$ is a positive constant.
\end{definition}

Random measurements have been proved to obey the RIP for Tucker and TT format tensors \cite{Rauhut17,qin2024quantum}.
\begin{theorem} (RIP for Tucker and TT decompositions \cite{Rauhut17,qin2024quantum})
\label{condition for the RIP TT and Tucker}
Suppose the linear map $\calA: \R^{d_{1}\times  \cdots \times d_{N}}\rightarrow \R^m$  is an $L$-subgaussian measurement ensemble.\footnote{A random variable $X$ is called $L$-subgaussian if there exists a constant $L > 0$ such that $\E e^{t X} \le e^{L^2t^2/2}
$ holds for all $t \in \R$. Typical cases include the Gaussian random variable and the Bernoulli random variable. We say that $\calA: \R^{d_1\times \cdots \times d_N} \rightarrow \R^m$  is an $L$-subgaussian measurement ensemble if all the elements of $\calA_k,k=1,\dots,m,$ are independent $L$-subgaussian random variables with mean zero and variance one~\cite{bierme2015modulus}.} Then, with probability at
least $1-\bar\epsilon$, $\calA$ satisfies the $\ol r$-RIP as in \eqref{RIP condition fro the Tucker TT sensing} for Tucker and TT decompositions given that
\begin{equation}
m \ge C \cdot \frac{1}{\delta_{\ol r}^2} \cdot \max\left\{ D\log N, \log(1/\bar\epsilon)\right \}
\label{eq:mrip Tucker TT}
\end{equation}
where for the Tucker format, we have $D = {\overline{r}^\text{tk}}^N + N\overline{d}\overline{r}^\text{tk}$ , while for the TT format, $D = N\overline{d}{\overline{r}^\text{tt}}^2$. In addition, $C$ is a universal constant depending only on $L$.
\end{theorem}

\paragraph{RCG condition for tensor sensing}
Based on the RIP, the following result establishes the RCG condition for the loss function in \eqref{General loss for G}.
\begin{theorem}
Suppose $\calA$ satisfies the $4\ol r$-RIP. Then we obtain the RCG condition for the loss function $G$. That is,
\begin{eqnarray}
    \label{RCG conclusion for Tu and TT}
    \<\nabla G(\calX)  - \nabla G(\calX^\star), \calX - \calX^\star \> \geq \frac{1 - \delta_{2\ol r}}{2}\|\calX - \calX^\star\|_F^2 + \frac{1 - \delta_{2\ol r}}{2(1 +\delta_{4\ol r})^2}\|\nabla G(\calX) - \nabla G(\calX^\star) \|_{F}^2
\end{eqnarray}
holds for any $\calX \in \calC$ where
 $\calC$ contains all Tucker/TT format tensors with rank $\ol r$, and
$\delta_{2\overline{r}}\leq \delta_{4\overline{r}}$.
\end{theorem}
The proof is given in {Appendix}~\ref{Proof of RCG in the tensor sensing}. Based on the analysis presented in \Cref{LocalConvergenceof General loss}  along with $\alpha = \frac{1 - \delta_{2\ol r}}{2}$ and $\beta = \frac{1 - \delta_{2\ol r}}{2(1 +\delta_{4\ol r})^2}$ of the RCG condition in \eqref{RCG conclusion for Tu and TT}, we can infer that the Riemannian regularity condition and the linear convergence rate of RGD can be achieved when $\calA$ fulfills the $4\ol r$-RIP condition. Our results demonstrate that when the number of measurements scales linearly with the degrees of freedom of a tensor decomposition, RGD can achieve a linear convergence rate with appropriate initialization. Compared with the theoretical result in our previous work \cite{qin2024guaranteed}, the  $(N+3)\ol r^\text{tt}$-RIP requirement has been relaxed to the $4\ol r^\text{tt}$-RIP.

\paragraph{Spectral initialization} To obtain a favorable initialization, we apply the spectral initialization, which has been widely utilized in tensor sensing \cite{Han20,ZhangISLET20,TongTensor21}. Specifically,  by one SVD-based method, we have
\begin{eqnarray}
    \label{spectral initialization in Tucker/TT tensor}
    \calX^{(0)} = [\mZ_1,\dots,\mZ_{\wt N}]={\rm SVD}\bigg(\frac{1}{m}\sum_{k=1}^m y_k\mathcal{A}_k\bigg).
\end{eqnarray}
Here, SVD refers to  high-order SVD (HOSVD) \cite{de2000multilinear} and TT SVD (TT-SVD) \cite{Oseledets11} for the Tucker and TT decompositions, respectively.

According to {Appendix} \ref{Proof of Tucker initialization} and \cite[Theorem 3]{qin2024guaranteed}, when $\calA$ satisfies the $3\overline{r}$-RIP condition, the spectral initialization satisfies the following property:
\begin{eqnarray}
    \label{TENSOR SENSING SPECTRAL INITIALIZATION1}
    \|\calX^{(0)}-\calX^\star\|_F\leq \delta_{3\ol r}(1+\sqrt{N})\|\calX^\star\|_F.
\end{eqnarray}
This indicates that if $\delta_{3\ol r}$ is sufficiently small, we can guarantee an arbitrarily accurate initialization.  In addition, we observe that the additional coefficient $O(\sqrt{N})$ arises because HOSVD and TT-SVD are only quasi-optimal approximations of Tucker and TT decompositions. However, if an optimal method for these decompositions were to be designed, this requirement could potentially be eliminated.

\paragraph{Tensor factorization and completion} $(i)$ The tensor factorization problem represents a special case of tensor sensing, wherein the operator $\calA$ corresponds to the identity operator. Clearly, when the loss function is $\frac{1}{2}\|\calX - \calX^\star\|_F^2$, the RCG condition is met with $\alpha = \beta = \frac{1}{2}$. $(ii)$ Another special case is tensor completion, where the goal is to reconstruct the entire tensor $\calX^\star$ from a subset of its entries.
Let $\Omega$ denote the indices of the  $m$ observed entries, and let $\calP_{\Omega}$ represent the corresponding projection operator that preserves the observed entries and sets the rest as zero. 
We define $J(\calX) = \frac{1}{2}\|\calP_{\Omega}(\calX - \calX^\star)\|_F^2$ and then obtain $\<\nabla J(\calX) - \nabla J(\calX^\star), \calX - \calX^\star  \> = \|\calP_{\Omega}(\calX - \calX^\star)\|_F^2$. In this case, the loss function $J$ satisfies the RCG condition for incoherent tensors \cite{TongTensor21,cai2022provable}. Specifically, based on the analysis in \cite[Lemma 5]{cai2022provable}, if $m\geq O(\sqrt{d_1\cdots d_N})$ sampled are taken uniformly with replacement,  and both $\calX,\calX^\star$ follow the incoherence property\footnote{
To enforce the incoherence condition of $\calX$ during the iteration, we can incorporate the truncation strategy  \cite[Algorithm 3]{cai2022provable} into the RGD \eqref{RGD for unified}.
However, we find that this truncation step is not necessary and results in only negligible difference in performance in practice, so we adopt the RGD \eqref{RGD for unified} in the experiments in the next section.
}, then we have $\<\calP_\Omega(\calX - \calX^\star), \calX - \calX^\star \>\geq \Omega(\frac{m}{d_1\cdots d_N})\|\calX - \calX^\star\|_F^2$ and $\|\calX - \calX^\star\|_F^2\geq \Omega(\frac{(d_1\cdots d_N)^2}{m^2})\|\calP_\Omega(\calX - \calX^\star)\|_F^2$.
Under these conditions, we can obtain $\<\nabla J(\calX) - \nabla J(\calX^\star), \calX - \calX^\star  \> = \<\calP_\Omega(\calX - \calX^\star), \calX - \calX^\star \>\geq \Omega(\frac{m}{d_1\cdots d_N})\|\calX - \calX^\star\|_F^2\geq \Omega(\frac{m}{2d_1\cdots d_N})\|\calX - \calX^\star\|_F^2 + \Omega(\frac{d_1\cdots d_N}{2m})\|\calP_\Omega(\calX - \calX^\star)\|_F^2$.

\section{Numerical Experiments}
\label{Numerical experiments}

In this section, we conduct  numerical experiments to evaluate the performance of the RGD algorithm.
We set the step size $\mu = 0.5$.
We generate the ground truth tensor $\calX^\star\in\R^{d_1\times\cdots \times d_N}$ in the TT format with a rank of $(r_1^\text{tt},\dots, r_{N-1}^\text{tt})$ by truncating a random Gaussian tensor using the TT-SVD. Similarly, we generate the ground truth tensor in the Tucker format with a rank of $(r_1^\text{tk},\dots,r_N^\text{tk})$ by truncating a random Gaussian tensor using the HOSVD. Then we normalize $\calX^\star$ to unit Frobenius norm, i.e., $\|\calX^\star\|_F=1$.
To simplify the selection of parameters, we let $d=d_1=\cdots=d_N$, $r^\text{tt}=r_1^\text{tt}=\cdots = r_{N-1}^\text{tt}$, and $r^\text{tk} = r_1^\text{tk} = \cdots = r_N^\text{tk}$. We conduct 20 Monte Carlo trials and take the average over the 20 trials.

\begin{figure}[htbp]
\centering
\subfigure[]{
\begin{minipage}[t]{0.31\textwidth}
\centering
\includegraphics[width=5.5cm]{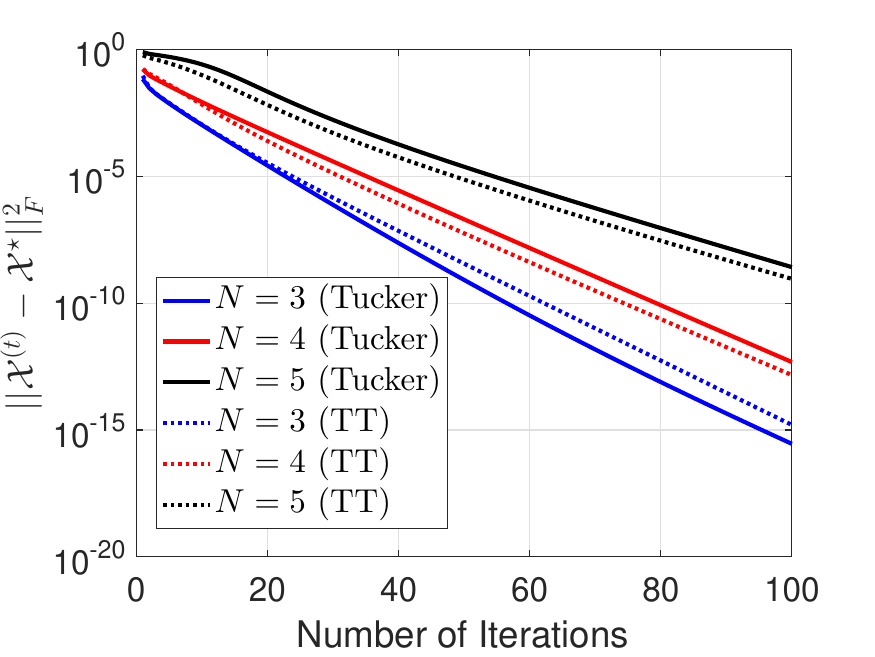}
\end{minipage}
\label{TT_sensing_N}
}
\subfigure[]{
\begin{minipage}[t]{0.31\textwidth}
\centering
\includegraphics[width=5.5cm]{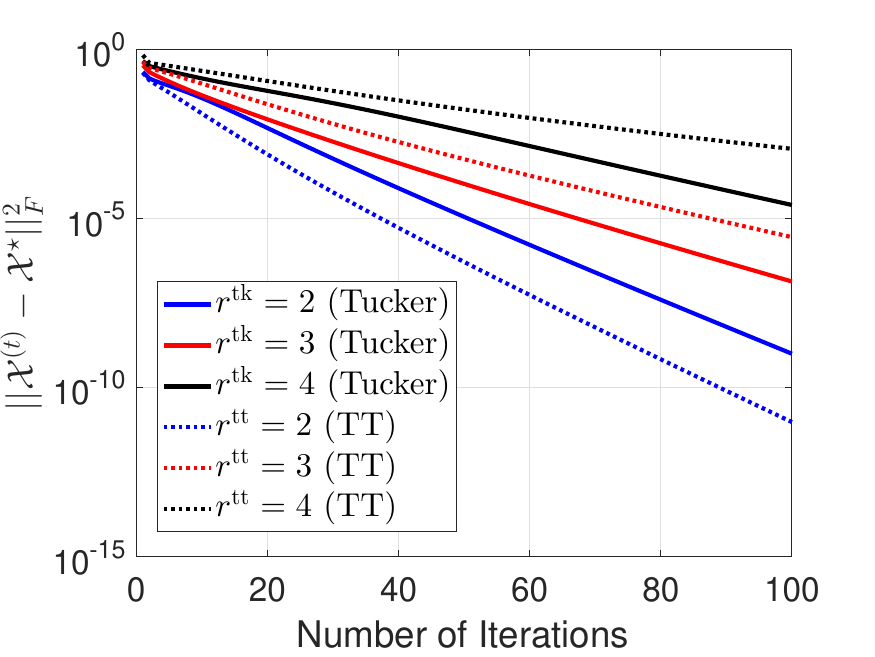}
\end{minipage}
\label{TT_sensing_r}
}
\subfigure[]{
\begin{minipage}[t]{0.31\textwidth}
\centering
\includegraphics[width=5.5cm]{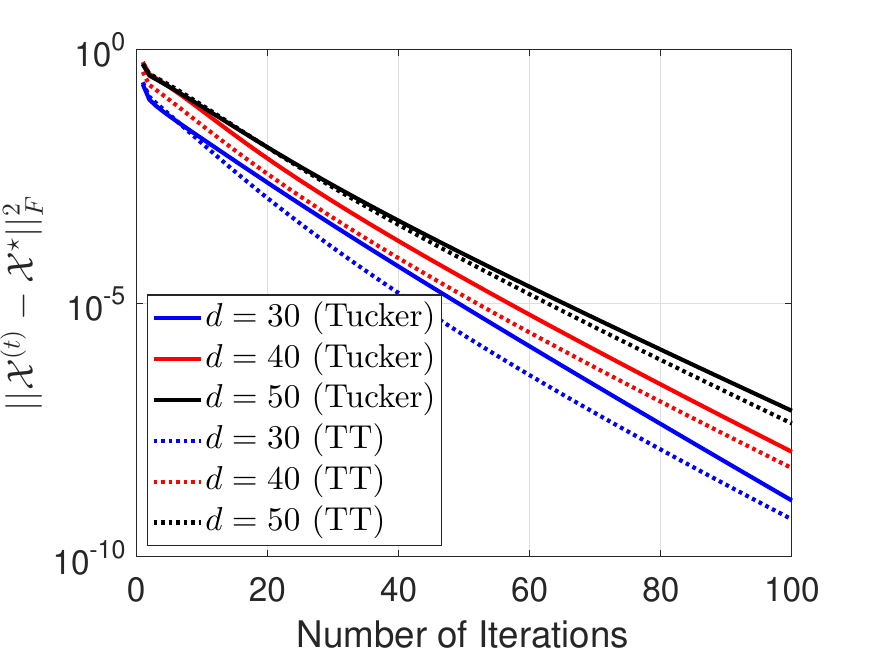}
\end{minipage}
\label{TT_sensing_d}
}
\caption{Performance comparison of the RGD, in the tensor sensing, (a) for different $N$ with $d = 10$, $r^\text{tk} = r^\text{tt} = 2$ and $m = 1000$, (b) for different $r^\text{tk}$ and $r^\text{tt}$ with $d = 50$, $N = 3$ and $m = 3000$, (c) for different $d$ with $N = 3$, $r^\text{tk} = r^\text{tt} = 2$ and $m = 1500$.}
\end{figure}

\paragraph{Tensor sensing} In the first experiment, we test the performance of RGD in the tensor sensing. First, we obtain measurements $y_i,i=1,\dots,m$ in \eqref{Definition of tensor sensing} from measurement operator $\calA_i$ which is a random tensor with independent entries generated from the normal distribution. Based on Figures~\ref{TT_sensing_N}-\ref{TT_sensing_d}, we can observe that the proposed RGD algorithm exhibits a linear convergence rate when the number of measurement operators $m$ is appropriately chosen. Furthermore, for a fixed $m$, as $N$, $d$, $r^\text{tt}/r^\text{tk}$ (indicating a larger RIP constant) increase, the convergence rate of the RGD algorithm diminishes.

\paragraph{Tensor completion} In the second experiment,  we consider the problem of tensor completion, with the goal of reconstructing the entire tensor $\calX^\star$ based on a subset of its entries.  Specifically, let $\Omega$ denote the indices of $m$ observed entries and let $\calP_{\Omega}$ denote the corresponding measurement operator that produces the observed measurements. Thus, we attempt to recover the underlying tensor by solving the following constrained factorized optimization problem
\begin{eqnarray}
    \label{Unified tensor completion loss function}
    \begin{split}
    \min_{\mZ_i, i\in[\wt N]} J(\mZ_1,\dots,\mZ_{\wt N}) & = \frac{1}{2}\|\calP_{\Omega}([\mZ_1,\dots,\mZ_{\wt N}] - [\mZ_1^\star,\dots,\mZ_{\wt N}^\star]) \|_F^2,\\
    \st \ \mZ_i^\top\mZ_i&=\mId, \  i \in[{\wt N}-1].
    \end{split}
\end{eqnarray}
We address this problem using the RGD algorithm. Furthermore, for an improved initialization, we utilize the sequential second-order moment method as proposed in \cite{cai2022provable}. From Figures~\ref{TT_completion_N}-\ref{TT_completion_d}, we observe that the RGD algorithm continues to exhibit a linear convergence rate. However, the convergence speed of the RGD algorithm decays as $N$, $d$, and $r^\text{tt}$ or $r^\text{tk}$ increase. Moreover, as demonstrated in \cite{cai2022provable}, the required number of samples $m$ grows exponentially with respect to $N$.

\begin{figure}[htbp]
\centering
\subfigure[]{
\begin{minipage}[t]{0.31\textwidth}
\centering
\includegraphics[width=5.5cm]{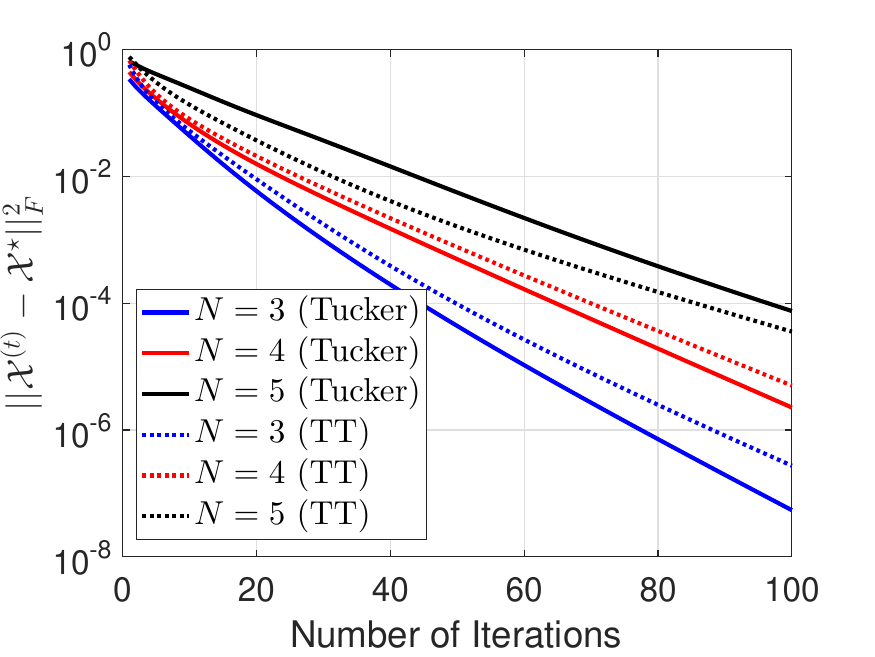}
\end{minipage}
\label{TT_completion_N}
}
\subfigure[]{
\begin{minipage}[t]{0.31\textwidth}
\centering
\includegraphics[width=5.5cm]{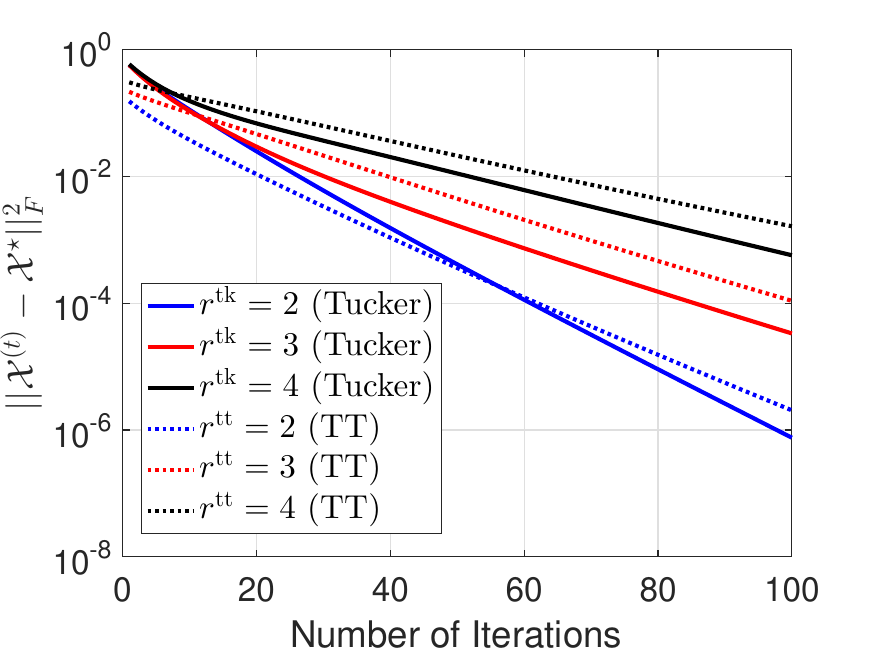}
\end{minipage}
\label{TT_completion_r}
}
\subfigure[]{
\begin{minipage}[t]{0.31\textwidth}
\centering
\includegraphics[width=5.5cm]{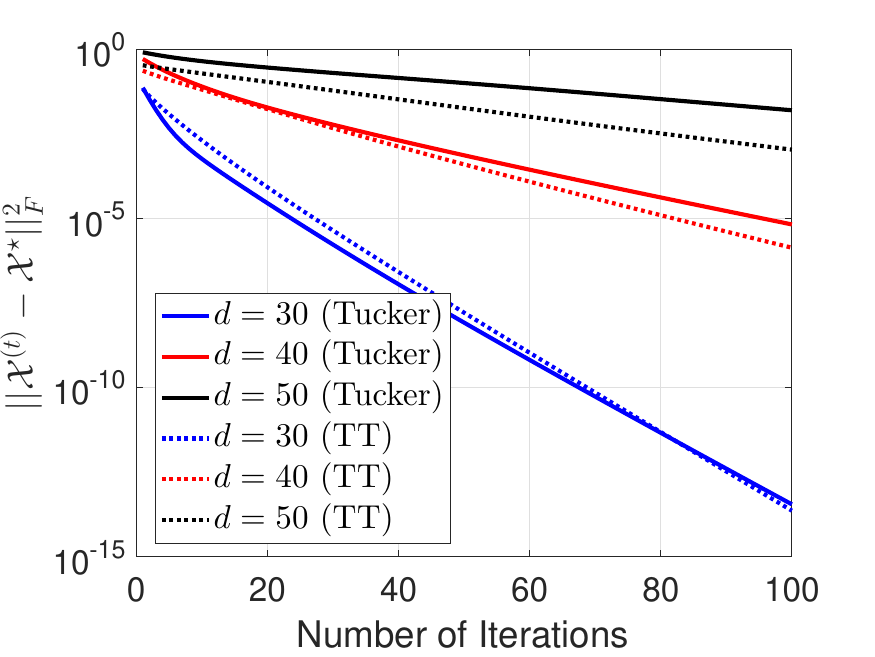}
\end{minipage}
\label{TT_completion_d}
}
\caption{Performance comparison of the RGD, in the tensor completion, (a) for different $N$ with $d = 10$, $r^\text{tk} = r^\text{tt} = 2$, $m = 500$ ($N=3$), $m=4000$ ($N=4$) and $m=30000$ ($N=5$), (b) for different $r^\text{tk}$ and $r^\text{tt}$ with $d = 50$, $N = 3$ and $m = 35000$, (c) for different $d$ with $N = 3$, $r^\text{tk} = r^\text{tt}= 2$ and $m = 20000$.}
\end{figure}

\section{Conclusion}
\label{conclusion}
This paper presents a comprehensive analysis of the convergence of the RGD algorithm on the Stiefel manifold for structured tensor recovery. We consider a class of functions and provide Riemannian regularity conditions. Based on these conditions, we further analyze the convergence rate of the RGD algorithm. Unlike previous works, our theoretical result offers a deterministic analysis for general orders. This demonstrates how the convergence speed and initial requirements change polynomially with increasing order. Additionally, we extend our analysis to tensor sensing, along with tensor factorization and tensor completion. Our experimental results confirm the effectiveness of our method.


\section{Acknowledgments}
\label{sec: ack}

We acknowledge funding support from NSF Grants No. CCF-2106834, CCF-2241298
and ECCS-2409701. We thank the Ohio Supercomputer Center for providing the computational resources needed in carrying out this work.

\appendices

\section{Technical tools used in the proofs}
\label{Technical tools used in proofs}

\begin{lemma}(\cite[Lemma 1]{LiSIAM21})
\label{NONEXPANSIVENESS PROPERTY OF POLAR RETRACTION_1}
Let $\mX\in\text{St}(m,n)$ and ${\bm \xi}\in\text{T}_{\mX} \text{St}$ be given. Consider the point $\mX^+=\mX+{\bm \xi}$. Then, the polar decomposition-based retraction satisfies $\text{Retr}_{\mX}(\mX^+)=\mX^+({\mX^+}^\top\mX^+)^{-\frac{1}{2}}$ and
\begin{eqnarray}
    \label{NONEXPANSIVENESS PROPERTY OF POLAR RETRACTION_2}
    \|\text{Retr}_{\mX}(\mX^+)-\overline\mX\|_F\leq\|\mX^+-\overline\mX\|_F=\|\mX+{\bm \xi}-\overline\mX\|_F, \  \forall\overline\mX\in\text{St}(m,n).
\end{eqnarray}
\end{lemma}

\begin{lemma}
\label{LGNTuckerTF1_5_1}
Let $\calX=[\![\calS; \mU_1, \dots, \mU_N ]\!]\in\R^{d_1\times\cdots\times d_N}$ be a rank-$(r_1^\text{tk},\dots, r_{N}^\text{tk})$ non-orthogonal Tucker format, where $\mU_i$ may not be orthonormal. We have
\begin{eqnarray}
    \label{Norm expansion of tucker tensor_2}
    &&\|\calX\|_F\leq \Pi_{i=1}^N\|\mU_i\|\cdot \|\calS\|_F,\\
    \label{Norm expansion of tucker tensor_3}
    &&\|\calX\|_F\leq \|\mU_1\|\cdots\|\mU_{i-1}\|\cdot \|\mU_{i+1}\|\cdots\|\mU_N\|\cdot\|\calS\|\cdot\|\mU_i\|_F, \ \ i\in[N].
\end{eqnarray}
\end{lemma}
\begin{proof}
By the definition of the matricization operator, we can derive
\begin{eqnarray}
    \label{Norm expansion of tucker tensor_4}
    \|\calX\|_F&\!\!\!\!=\!\!\!\!&\|\mU_i\calM_i(\calS)(\mU_{i-1}\otimes\cdots\otimes\mU_1\otimes\mU_N\otimes\cdots\otimes\mU_{i+1})^\top\|_F\nonumber\\
    &\!\!\!\!\leq\!\!\!\!&\|\mU_i\| \cdot \|\mU_{i-1}\otimes\cdots\otimes\mU_1\otimes\mU_N\otimes\cdots\otimes\mU_{i+1}\| \cdot \|\calM_i(\calS)\|_F\nonumber\\
    &\!\!\!\!=\!\!\!\!&\|\mU_1\|\cdot\|\mU_2\|\cdots\|\mU_N\|\cdot\|\calS\|_F,
\end{eqnarray}
where in the second equation, we use $\|\mA\otimes\mB\|=\|\mA\|\cdot\|\mB\|$ for any matrices $\mA$ and $\mB$.

Similarly, we have
\begin{eqnarray}
    \label{Norm expansion of tucker tensor_5}
    \|\calX\|_F&\!\!\!\!=\!\!\!\!&\|\mU_i\calM_i(\calS)(\mU_{i-1}\otimes\cdots\otimes\mU_1\otimes\mU_N\otimes\cdots\otimes\mU_{i+1})^\top\|_F\nonumber\\
    &\!\!\!\!\leq\!\!\!\!&\|\calM_i(\calS)\| \cdot \|\mU_{i-1}\otimes\cdots\otimes\mU_1\otimes\mU_N\otimes\cdots\otimes\mU_{i+1}\|\cdot\|\mU_i\|_F\nonumber\\
    &\!\!\!\!=\!\!\!\!& \|\mU_1\|\cdots\|\mU_{i-1}\| \cdot \|\mU_{i+1}\|\cdots\|\mU_N\| \cdot \|\calS\| \cdot \|\mU_i\|_F,i\in[N].
\end{eqnarray}

\end{proof}

\begin{lemma} (\cite{cai2022provable,Han20})
\label{left ortho upper bound}
For any two matrices $\mX, \mX^\star$ with rank $r $, let $\mU \mSigma \mV^T$ and $\mU^\star \mSigma^\star {\mV^\star}^T$  respectively represent the compact singular value decompositions (SVDs) of $\mX$ and $\mX^\star$. Supposing that $\mR = \argmin_{\wt{\mR}\in\O^{r\times r}}\|\mU - \mU^\star\wt{\mR}  \|_F $, we have
\begin{eqnarray}
    \label{relationship between left O with full tensor SVD}
    \|\mU - \mU^\star\mR \|_F \leq \frac{2\|\mX-\mX^\star\|_F}{\sigma_r(\mX^\star)}.
\end{eqnarray}
\end{lemma}

\begin{lemma}
\label{Lower bound of two distance in the tucker tensor_1}
For any two Tucker formats $\calX=[\![\calS; \mU_1, \dots, \mU_N ]\!]$ and $\calX^\star=[\![\calS^\star; \mU_1^\star, \dots, \mU_N^\star ]\!]$, we assume that $\overline{\sigma}_{\text{tk}}(\calX) =\|\calS\|\leq \frac{3\overline{\sigma}_{\text{tk}}(\calX^\star)}{2}$ and then have
\begin{eqnarray}
    \label{Lower bound of two distance in the tucker tensor_2}
    &&\|\calX-\calX^\star\|_F^2\geq\frac{1}{(9N^2+14N+1)\kappa^2_\text{tk}(\calX^\star)}\text{dist}_\text{tk}^2(\{\mU_i,\calS \}, \{\mU_i^\star,\calS^\star \}),\\
    \label{Upper bound of two distance in the tucker tensor_2}
    &&\|\mathcal{X}-\mathcal{X}^\star\|_F^2\leq\frac{9(N+1)}{4}\text{dist}_\text{tk}^2(\{\mU_i,\calS \}, \{\mU_i^\star,\calS^\star \}),
\end{eqnarray}
where $\text{dist}_\text{tk}^2(\{\mU_i,\calS \}, \{\mU_i^\star,\calS^\star \})=\min_{\substack{ \mR_i\in\O^{r_i^\text{tk}\times r_i^\text{tk}}, \\ i\in[N]  }}\sum_{i=1}^N\overline{\sigma}^2_{\text{tk}}(\calX^\star)\|\mU_i-\mU_i^\star\mR_i\|_F^2+\|\calS-[\![\calS^\star; \mR_1^\top,\dots, \mR_N^\top  ]\!]  \|_F^2$.
\end{lemma}

\begin{proof}
According to \Cref{left ortho upper bound}, we have
\begin{eqnarray}
    \label{Lower bound of two distance in the tucker tensor_3}
    \|\mU_i-\mU_i^\star\mR_i\|_F^2\leq\frac{4\|\calX-\calX^\star\|_F^2}{\underline{\sigma}_{\text{tk}}^2(\calX^\star)}, \ \  i\in[N]
\end{eqnarray}
for any orthonormal matrices $R_i\in\O^{r_i^\text{tk}\times r_i^\text{tk}}, i\in[N]$. In addition, we can derive
\begin{eqnarray}
    \label{Lower bound of two distance in the tucker tensor_4}
    &\!\!\!\!\!\!\!\!&\|\calS-[\![\calS^\star; \mR_1^\top,\dots, \mR_N^\top  ]\!]\|_F^2\nonumber\\
    &\!\!\!\!=\!\!\!\!& \bigg\|\sum_{i=1}^N\calS \times_1 \mU_1 \times_2\cdots \times_{i-1} \mU_{i-1} \times_i (\mU_i^\star\mR_i-\mU_i)\times_{i+1}\mU_{i+1}^\star\mR_{i+1} \times_{i+2} \cdots \times_N\mU_N^\star\mR_N +\calX-\calX^\star \bigg\|_F^2\nonumber\\
    &\!\!\!\!\leq\!\!\!\!&(N+1)\bigg(\frac{9\overline{\sigma}_{\text{tk}}^2(\calX^\star)}{4}\sum_{i=1}^N\|\mU_i-\mU_i^\star\mR_i\|_F^2+\|\calX-\calX^\star\|_F^2\bigg)\nonumber\\
    &\!\!\!\!\leq\!\!\!\!&(9N^2+10N+1)\kappa^2_\text{tk}(\calX^\star)\|\calX-\calX^\star\|_F^2,
\end{eqnarray}
where \Cref{LGNTuckerTF1_5_1} and \eqref{Lower bound of two distance in the tucker tensor_3} are respectively used in the first and second inequalities.

Combing \eqref{Lower bound of two distance in the tucker tensor_3} and \eqref{Lower bound of two distance in the tucker tensor_4}, we have
\begin{eqnarray}
    \label{Lower bound of two distance in the tucker tensor_5}
    \text{dist}_\text{tk}^2(\{\mU_i,\calS \}, \{\mU_i^\star,\calS^\star \})&\!\!\!\!=\!\!\!\!&\min_{\substack{ R_i\in\O^{r_i\times r_i}, \\ i\in [N]}}\sum_{i=1}^N\overline{\sigma}_{\text{tk}}^2(\calX^\star)\|\mU_i-\mU_i^\star\mR_i\|_F^2+\|\calS-[\![\calS^\star; \mR_1^\top,\dots, \mR_N^\top  ]\!]  \|_F^2\nonumber\\
    &\!\!\!\!\leq\!\!\!\!&4N\kappa^2_\text{tk}(\calX^\star)\|\calX-\calX^\star\|_F^2+(9N^2+10N+1)\kappa^2_\text{tk}(\calX^\star)\|\calX-\calX^\star\|_F^2\nonumber\\
    &\!\!\!\!=\!\!\!\!&(9N^2+14N+1)\kappa^2_\text{tk}(\calX^\star)\|\calX-\calX^\star\|_F^2.
\end{eqnarray}

On the other hand, we can expand $\|\mathcal{X}-\mathcal{X}^\star\|_F^2$  and then obtain
\begin{eqnarray}
    \label{Upper bound of two distance in the tucker tensor_3}
    &\!\!\!\!\!\!\!\!&\|\mathcal{X}-\mathcal{X}^\star\|_F^2\nonumber\\
    &\!\!\!\!=\!\!\!\!&\|\sum_{i=1}^N\calS \times_1 \mU_1^\star\mR_1 \times_2\cdots \times_{i-1} \mU_{i-1}^\star\mR_{i-1} \times_i (\mU_i-\mU_i^\star\mR_i)\times_{i+1}\mU_{i+1} \times_{i+1} \cdots \times_N\mU_N \nonumber\\
    &\!\!\!\!\!\!\!\!&+(\calS-[\![\calS^\star; \mR_1^\top,\dots, \mR_N^\top  ]\!])\times_1\mU_1^\star\mR_1\times_2\cdots\times_N \mU_N^\star\mR_N  \|_F^2\nonumber\\
    &\!\!\!\!\leq\!\!\!\!&(N+1)\bigg(\frac{9\overline{\sigma}_{\text{tk}}^2(\calX^\star)}{4}\sum_{i=1}^N\|\mU_i-\mU_i^\star\mR_i\|_F^2+\|\calS-[\![\calS^\star; \mR_1^\top,\dots, \mR_N^\top  ]\!]\|_F^2\bigg)\nonumber\\
    &\!\!\!\!\leq\!\!\!\!&\frac{9(N+1)}{4}\text{dist}_\text{tk}^2(\{\mU_i,\calS \}, \{\mU_i^\star,\calS^\star \}).
\end{eqnarray}

\end{proof}

\begin{lemma}
\label{TRANSFORMATION OF N+1 to N(N+1)}
Given that $\calX = [\![\calS; \mU_1, \dots, \mU_N ]\!]$ and  $\calX^\star = [\![ [\![\calS^\star; \mR_1^\top,\dots, \mR_N^\top  ]\!];   \mU_1^\star\mR_1, \dots, \mU_N^\star\mR_N    ]\!]$ for any orthonormal matrices $\mR_i\in\O^{r_i^\text{tk}\times r_i^\text{tk}}, i\in[N]$, we have
\begin{eqnarray}
    \label{TRANSFORMATION OF Tucker format}
    &\!\!\!\!\!\!\!\!&\calX^\star  -[\![ \calS^\star ;   \mU_1^\star\mR_1^\top, \dots, \mU_N^\star\mR_N^\top    ]\!] + \sum_{i=1}^N [\![\calS; \mU_1, \dots,\mU_{i-1}, \mU_i -  \mU_i^\star\mR_i,  \mU_{i+1}, \dots, \mU_N ]\!]\nonumber\\
    &\!\!\!\!= \!\!\!\!&\sum_{i=1}^{N-1}\sum_{j=i+1}^N [\![\calS; \mU_1, \dots,\mU_{i-1}, \mU_i -  \mU_i^\star\mR_i, \mU_{i+1}^\star\mR_{i+1}, \dots, \mU_{j-1}^\star\mR_{j-1}, \mU_j -  \mU_j^\star\mR_j,  \mU_{j+1}, \dots, \mU_N ]\!]\nonumber\\
    &\!\!\!\!\!\!\!\!&+ \sum_{k=1}^N [\![\calS-[\![\calS^\star; \mR_{1}^\top,\dots, \mR_{N}^\top  ]\!]; \mU_1, \dots,   \mU_{k-1}, (\mU_k
    -\mU_k^\star\mR_{k}),   \mU_{k+1}^\star{\mR_{k+1}}, \dots,  \mU_{N}^\star{\mR_N}    ]\!],
\end{eqnarray}
where the right-hand side of \eqref{TRANSFORMATION OF Tucker format} contains a total of $\frac{N(N+1)}{2}$ terms.

\end{lemma}

\begin{proof}
Firstly, we expand $[\![\calS; \mU_1, \dots,\mU_{i-1}, \mU_i -  \mU_i^\star\mR_i,  \mU_{i+1}, \dots, \mU_N ]\!]$ as
\begin{eqnarray}
    \label{Expansion of the first term}
    &\!\!\!\!\!\!\!\!&\sum_{i=1}^N[\![\calS; \mU_1, \dots,\mU_{i-1}, \mU_i -  \mU_i^\star\mR_i,  \mU_{i+1}, \dots, \mU_N ]\!]\nonumber\\
    &\!\!\!\! = \!\!\!\!&  \sum_{i=1}^{N-1}\sum_{j=i+1}^N [\![\calS; \mU_1, \dots,\mU_{i-1}, \mU_i -  \mU_i^\star\mR_i, \mU_{i+1}^\star\mR_{i+1}, \dots, \mU_{j-1}^\star\mR_{j-1}, \mU_j -  \mU_j^\star\mR_j,  \mU_{j+1}, \dots, \mU_N ]\!]\nonumber\\
    &\!\!\!\!\!\!\!\!& + \sum_{i=1}^N[\![\calS -[\![\calS^\star; \mR_{1}^\top,\dots, \mR_{N}^\top  ]\!]; \mU_1, \dots,\mU_{i-1}, \mU_i -  \mU_i^\star\mR_i, \mU_{i+1}^\star\mR_{i+1}, \dots, \mU_{N-1}^\star\mR_{N-1},   \mU_N^\star\mR_N ]\!]\nonumber\\
    &\!\!\!\!\!\!\!\!& + \sum_{i=1}^N[\![ [\![\calS^\star; \mR_{1}^\top,\dots, \mR_{N}^\top  ]\!]; \mU_1, \dots,\mU_{i-1}, \mU_i -  \mU_i^\star\mR_i, \mU_{i+1}^\star\mR_{i+1}, \dots, \mU_{N-1}^\star\mR_{N-1},   \mU_N^\star\mR_N ]\!].
\end{eqnarray}

Then we need to compute
\begin{eqnarray}
    \label{zero summation of additional terms}
    &\!\!\!\!\!\!\!\!&\sum_{i=1}^N [\![ [\![\calS^\star; \mR_{1}^\top,\dots, \mR_{N}^\top  ]\!]; \mU_1, \dots,\mU_{i-1}, \mU_i -  \mU_i^\star\mR_i, \mU_{i+1}^\star\mR_{i+1}, \dots, \mU_{N-1}^\star\mR_{N-1},   \mU_N^\star\mR_N ]\!]\nonumber\\
    &\!\!\!\!\!\!\!\!& +\calX^\star  -[\![ \calS^\star ;   \mU_1^\star\mR_1^\top, \dots, \mU_N^\star\mR_N^\top    ]\!]\nonumber\\
    &\!\!\!\! = \!\!\!\!& \sum_{i=1}^N [\![ [\![\calS^\star; \mR_{1}^\top,\dots, \mR_{N}^\top  ]\!]; \mU_1, \dots,\mU_{i-1}, \mU_i -  \mU_i^\star\mR_i, \mU_{i+1}^\star\mR_{i+1}, \dots, \mU_{N-1}^\star\mR_{N-1},   \mU_N^\star\mR_N ]\!]\nonumber\\
    &\!\!\!\!\!\!\!\!& + [\![ [\![\calS^\star; \mR_1^\top,\dots, \mR_N^\top  ]\!];   \mU_1^\star\mR_1, \dots, \mU_N^\star\mR_N    ]\!]  -[\![ \calS^\star ;   \mU_1^\star\mR_1^\top, \dots, \mU_N^\star\mR_N^\top    ]\!]\nonumber\\
    &\!\!\!\! = \!\!\!\!& [\![ [\![\calS^\star; \mR_1^\top,\dots, \mR_N^\top  ]\!];   \mU_1^\star, \dots, \mU_N^\star    ]\!]  -[\![ \calS^\star ;   \mU_1^\star\mR_1^\top, \dots, \mU_N^\star\mR_N^\top    ]\!]\nonumber\\
    &\!\!\!\! = \!\!\!\!& [\![ \calS^\star ;   \mU_1^\star\mR_1^\top, \dots, \mU_N^\star\mR_N^\top    ]\!] - [\![ \calS^\star ;   \mU_1^\star\mR_1^\top, \dots, \mU_N^\star\mR_N^\top    ]\!] =0.
\end{eqnarray}

This completes the proof.

\end{proof}

\begin{lemma}(\cite[Lemma 4]{qin2024guaranteed})
\label{TRANSFORMATION OF NPLUS2 VARIABLES_1}
For any $\mA_i,\mA^\star_i\in\R^{r_{i-1}\times r_i}, i \in[N]$, we have
\begin{eqnarray}
    \label{TRANSFORMATION OF NPLUS2 VARIABLES_2}  &\!\!\!\!\!\!\!\!&\mA_1^\star\cdots \mA_N^\star-\mA_1\dots\mA_{N-1}\mA_N^\star + \sum_{i=1}^{N-1} \mA_1 \cdots \mA_{i-1} (\mA_i- \mA_i^\star) \mA_{i+1} \cdots \mA_N\nonumber\\
    &\!\!\!\!=\!\!\!\!& \sum_{i=1}^{N-1}\sum_{j=i+1}^N \mA_1\cdots \mA_{i-1}(\mA_i-\mA_i^\star)\mA_{i+1}^\star \cdots \mA_{j-1}^\star (\mA_{j}-\mA_{j}^\star)\mA_{j+1}\cdots \mA_N,
\end{eqnarray}
where the right-hand side of \eqref{TRANSFORMATION OF NPLUS2 VARIABLES_2} contains a total of $\frac{N(N-1)}{2}$ terms.

\end{lemma}

\begin{lemma}(\cite[Lemma 1]{qin2024guaranteed})
\label{LOWER BOUND OF TWO DISTANCES}
For any two left orthogonal TT formats $\calX = [\mX_1,\dots, \mX_N] $ and $\calX^\star = [\mX_1^\star,\dots, \mX_N^\star]$  with rank $(r_1^\text{tt},\dots,r_{N-1}^\text{tt})$,  we assume that $\ol\sigma^2_\text{tt}(\calX)\leq\frac{9\ol\sigma^2_\text{tt}(\calX^\star)}{4}$ and then have
\begin{eqnarray}
    \label{LOWER BOUND OF TWO DISTANCES_1}
    &&\|\calX-\calX^\star\|_F^2\geq\frac{1}{8(N+1+\sum_{i=2}^{N-1}r_i^\text{tt})\kappa_\text{tt}^2(\calX^\star)}\text{dist}^2_\text{tt}(\{\mX_i  \},\{ \mX_i^\star \}),\\
    \label{UPPER BOUND OF TWO DISTANCES_1 main paper}
    &&\|\calX-\calX^\star\|_F^2\leq\frac{9N}{4}\text{dist}^2_\text{tt}(\{\mX_i  \},\{ \mX_i^\star \}),
\end{eqnarray}
where $\text{dist}^2_\text{tt}(\{\mX_i  \},\{ \mX_i^\star \})=\min_{\mR_i\in\O^{r_i^\text{tt}\times r_i^\text{tt}}, \atop i \in [N-1]}\sum_{i=1}^{N-1} \ol\sigma^2_\text{tt}(\calX^\star)\|L({\mX}_i)-L_{\mR}({\mX}_i^\star)\|_F^2 + \|L({\mX}_N)-L_{\mR}({\mX}_N^\star)\|_2^2$.
\end{lemma}

Then, we introduce a new operation related to the multiplication of submatrices within the left unfolding matrices $L(\mX_i)=\begin{bmatrix}\mX_i(1) \\ \vdots\\  \mX_i(d_i) \end{bmatrix}\in\R^{(r_{i-1}d_i) \times r_i}, i \in [N]$. For simplicity, we will only consider the case $d_i=2$, but extending to the general case is straightforward.

Let $\mA=\begin{bmatrix}\mA_1 \\ \mA_2 \end{bmatrix}$ and $\mB=\begin{bmatrix}\mB_1 \\ \mB_2 \end{bmatrix}$ be two block matrices, where $\mA_i\in\R^{r_1\times r_2}$ and $\mB_i\in\R^{r_2\times r_3}$ for $i=1,2$. We introduce the notation $\ol \otimes$ to represent the Kronecker product between submatrices in the two block matrices, as an alternative to the standard Kronecker product based on element-wise multiplication.
Specifically, we define $\mA\ol \otimes\mB$ as
\begin{eqnarray}
    \label{KRONECKER PRODUCT VECTORIZATION}
    &&\mA\ol \otimes\mB=\begin{bmatrix}\mA_1 \\ \mA_2 \end{bmatrix}\ol \otimes\begin{bmatrix}\mB_1 \\ \mB_2 \end{bmatrix}=\begin{bmatrix}\mA_1\mB_1 \\ \mA_2\mB_1 \\ \mA_1\mB_2 \\ \mA_2\mB_2 \end{bmatrix}.
\end{eqnarray}

According to \cite[Lemma 2]{qin2024guaranteed}, we can conclude that for any left orthogonal TT format $\calX^\star = [\mX_1^\star,\dots, \mX_N^\star]$, we have
\begin{eqnarray}
    \label{KRONECKER PRODUCT VECTORIZATION11}
    &&\|\calX^\star\|_F = \|\text{vec}(\calX^\star)\|_2=\|L(\mX_1^\star) \ol \otimes \cdots \ol \otimes L(\mX_N^\star)\|_2 = \|L(\mX_N^\star)\|_2, \\
    \label{KRONECKER PRODUCT VECTORIZATION11 - 2}
    &&\|L(\mX_i^\star) \ol \otimes \cdots  \ol \otimes L(\mX_j^\star)\|\leq \Pi_{l = i}^{j}\|L(\mX_l^\star)\| = 1, \ \ i\leq j, \ \ \forall i, j\in[N-1], \\
        \label{KRONECKER PRODUCT VECTORIZATION11 - 3}
    &&\|L(\mX_i^\star) \ol \otimes \cdots  \ol \otimes L(\mX_j^\star)\|_F\leq \Pi_{l = i}^{j-1}\|L(\mX_l^\star)\| \|L(\mX_j^\star)\|_F, \ \ i\leq j, \ \ \forall i, j\in[N-1].
\end{eqnarray}

In addition, according to \eqref{KRONECKER PRODUCT VECTORIZATION}, we note that each row in $L(\mX_1^\star) \ol \otimes \cdots \ol \otimes L(\mX_i^\star)$ can be represented as
\begin{eqnarray}
    \label{each row in Kronecker product L}
    (L(\mX_1^\star) \ol \otimes \cdots \ol \otimes L(\mX_i^\star))(s_1\cdots s_i,:) &\!\!\!\!=\!\!\!\!&(L(\mX_1^\star) \ol \otimes \cdots \ol \otimes L(\mX_i^\star))(s_1 + d_1(s_2-2) + \cdots + d_1\cdots d_{i-1}(s_i-1),:) \nonumber\\
    &\!\!\!\!=\!\!\!\!& \mX_1(s_1)\cdots \mX_i(s_i).
\end{eqnarray}

\section{Riemannian regularity condition for Tucker format}
\label{proof of regularity condition of general loss in Tucker format}

In this section, we first define the loss function for the Tucker format as following:
\begin{align}
    \label{Loss function of teucker tensor general in intro}
    \begin{split}
\min_{\mbox{\tiny$\begin{array}{c}
     \mU_i\in\R^{d_i\times r_i^\text{tk}}, i\in[N],\\
     \calS\in\R^{r_1^\text{tk}\times  \cdots \times r_N^\text{tk}}\end{array}$}}\!\! H_1(\mU_1,\dots,\mU_N, \calS) & = h([\![\calS; \mU_1, \dots, \mU_N ]\!]),\\
\st \ \mU_i^\top\mU_i&=\mId_{r_i^\text{tk}}, \  i\in [N].
    \end{split}
\end{align}

Then, to avoid confusion we define
\begin{eqnarray}
    \label{The distance of factors in Tucker tensor appe}
    \text{dist}_\text{tk}^2(\{\mU_i,\calS \}, \{\mU_i^\star,\calS^\star \})=\min_{\substack{ \mR_i\in\O^{r_i^\text{tk}\times r_i^\text{tk}}, \\ i\in[N]  }}\sum_{i=1}^N\overline{\sigma}^2_{\text{tk}}(\calX^\star)\|\mU_i-\mU_i^\star\mR_i\|_F^2+\|\calS-[\![\calS^\star; \mR_1^\top,\dots, \mR_N^\top  ]\!]  \|_F^2.
\end{eqnarray}
Now we can establish the Riemannian regularity condition as follows:

\begin{lemma}
\label{regularity condition of general loss in the Tucker format}
Let a rank-$(r_1^\text{tk},\dots, r_{N}^\text{tk})$ Tucker format $\calX^\star=[\![\calS^\star; \mU_1^\star, \dots, \mU_N^\star ]\!]$ be a target tensor of a loss function $h$ in \eqref{Loss function of teucker tensor general in intro} which satisfies the RCG condition with   the set $\calC_\text{tk}(c^\text{tk}a_1^\text{tk})$ as following:
\begin{eqnarray}
    \label{definition of the intial set of generel loss in Tucker}
    \calC_\text{tk}(a_1^\text{tk}) : =\bigg\{\calX=[\![\calS; \mU_1, \dots, \mU_N]\!]: \|\calX - \calX^\star\|_F^2\leq c^\text{tk} a_1^\text{tk} \bigg\},
\end{eqnarray}
where $a_1^\text{tk} = \frac{8\alpha\beta\underline{\sigma}_{\text{tk}}^2(\calX^\star)}{9N(N+2)(9N^2+14N+1)}$ and $c^\text{tk} = \frac{1}{(9N^2+14N+1)\kappa^2_\text{tk}(\calX^\star)}$. The loss function $H_1$ in \eqref{Loss function of teucker tensor general in intro} satisfies a Riemannian regularity condition, denoted by $\text{RRC}(a_2^\text{tk},a_3^\text{tk})$, if for all  $\{\mU_i,\calS \}\in\{\{\mU_i,\calS \}: \text{dist}_\text{tk}^2(\{\mU_i,\calS \}, \{\mU_i^\star,\calS^\star \})\leq a_1^\text{tk}  \}$ in the Tucker format,  the following inequality holds:
\begin{align}
    \label{definition of regularity condition for tucker gene loss}
    &\hspace{0.4cm} \sum_{i=1}^N\bigg\<\mU_i-\mU_i^\star\mR_i, \calP_{\text{T}_{\mU_i} \text{St}}(\nabla_{\mU_i}H_1(\mU_1,\dots,\mU_N, \calS ))\bigg\>+\bigg\<\calS-[\![\calS^\star; \mR_{1}^\top,\dots, \mR_{N}^\top  ]\!], \nabla_{\calS}H_1(\mU_1,\dots,\mU_N, \calS ) \bigg\>\nonumber\\
    &\geq a_2^\text{tk}\text{dist}_\text{tk}^2(\{\mU_i,\calS \}, \{\mU_i^\star,\calS^\star \})
    +a_3^\text{tk}\bigg(\frac{1}{\overline{\sigma}_{\text{tk}}^2(\calX^\star)}\sum_{i=1}^N\|\calP_{\text{T}_{\mU_i} \text{St}}(\nabla_{\mU_i}H_1(\mU_1,\dots, \mU_N,\calS ))\|_F^2+\|\nabla_{\calS}H_1(\mU_1,\dots,\mU_N, \calS )\|_F^2  \bigg),
\end{align}
where $a_2^\text{tk} = \frac{\alpha}{2(9N^2+14N+1)\kappa^2_\text{tk}(\calX^\star)}$, $a_3^\text{tk} = \frac{\beta}{9N+4}$ and $(\mR_1,\dots, \mR_{N}) = \argmin_{\wt\mR_i\in\O^{r_i^\text{tk}\times r_i^\text{tk}}, i\in[{N}]}\text{dist}_\text{tk}^2(\{\mU_i,\calS \}, \{\mU_i^\star,\calS^\star \})$.
\end{lemma}

\begin{proof}
Before proving \Cref{regularity condition of general loss in the Tucker format}, we first introduce one useful property of $\calS$. Specifically, we get
\begin{eqnarray}
    \label{Upper bound of calS}
    \|\calS\|^2 &\!\!\!\!\leq\!\!\!\!& 2\|[\![\calS^\star; \mR_{1}^\top,\dots, \mR_{N}^\top  ]\!]\|^2+2\|\calS-[\![\calS^\star; \mR_{1}^\top,\dots, \mR_{N}^\top  ]\!]\|^2\nonumber\\
    &\!\!\!\!\leq \!\!\!\!& 2\overline{\sigma}_{\text{tk}}^2(\calX^\star) + 2\text{dist}_\text{tk}^2(\{\mU_i,\calS \}, \{\mU_i^\star,\calS^\star \})\nonumber\\
    &\!\!\!\!\leq\!\!\!\!& 2\overline{\sigma}_{\text{tk}}^2(\calX^\star) + \frac{16\alpha\beta\underline{\sigma}_{\text{tk}}^2(\calX^\star)}{9N(N+2)(9N^2+14N+1)}< \frac{9\overline{\sigma}_{\text{tk}}^2(\calX^\star)}{4},
\end{eqnarray}
where we use the inequality $\alpha\beta\leq \frac{1}{4}$ \cite{Han20}.

In addition, the gradients $\{\nabla_{\mU_i}H_1(\mU_1,\dots,\mU_N, \calS )\}$ and $\nabla_{\calS}H_1(\mU_1,\dots,\mU_N, \calS )$ have been defined as follows:
\begin{eqnarray}
    \label{gradients of Riemannain gradient descent algorithm factorization 1toN general}
    &&\hspace{-0.5cm}\nabla_{\mU_i}H_1(\mU_1,\dots,\mU_N, \calS )=\calM_i(\nabla_{\calX} h(\calX))(\mU_{i-1}\otimes\cdots\otimes\mU_1\otimes\mU_N\otimes\cdots\otimes\mU_{i+1})\calM_i(\calS)^\top, i \in [N],\\
    \label{gradients of Riemannain gradient descent algorithm factorization S general}
    &&\hspace{-0.5cm}\nabla_{\calS}H_1(\mU_1,\dots,\mU_N, \calS )=\nabla_{\calX} h(\calX)\times_1 \mU_1^\top\times_2  \cdots \times_N \mU_N^\top.
\end{eqnarray}

\paragraph{Upper bound of the squared term in \eqref{definition of regularity condition for tucker gene loss}} We first apply norm inequalities to derive
\begin{eqnarray}
    \label{Upper bound of squared term in the tensor general Riemannian 1ToN}
    &\!\!\!\!\!\!\!\!&\|\nabla_{\mU_i}H_1(\mU_1,\dots,\mU_N, \calS )\|_F^2\nonumber\\
    &\!\!\!\!\leq\!\!\!\!&\|\nabla_{\calX} h(\calX)-\nabla_{\calX} h(\calX^\star)\|_F^2\|\calS\|^2\nonumber\\
    &\!\!\!\!\leq\!\!\!\!&\frac{9\overline{\sigma}_{\text{tk}}^2(\calX^\star)}{4}\|\nabla_{\calX} h(\calX)-\nabla_{\calX} h(\calX^\star)\|_F^2, i\in[N],
\end{eqnarray}
where we use $\nabla_{\calX} h(\calX^\star) = 0$. Similarly, we can also obtain
\begin{eqnarray}
    \label{Upper bound of squared term in the tensor general Riemannian S}
    \|\nabla_{\calS}H_1(\mU_1,\dots,\mU_N, \calS )\|_F^2\leq\|\nabla_{\calX} h(\calX)-\nabla_{\calX} h(\calX^\star)\|_F^2.
\end{eqnarray}

Combining \eqref{Upper bound of squared term in the tensor general Riemannian 1ToN} and \eqref{Upper bound of squared term in the tensor general Riemannian S}, we can derive
\begin{eqnarray}
    \label{Upper bound of squared term in the tensor general Riemannian Conclusion}
    &\!\!\!\!\!\!\!\!&\sum_{i=1}^N\frac{1}{\overline{\sigma}_{\text{tk}}^2(\calX^\star)}\|\calP_{\text{T}_{\mU_i} \text{St}}(\nabla_{\mU_i}H_1(\mU_1,\dots, \mU_N,\calS))\|_F^2+\|\nabla_{\calS}H_1(\mU_1,\dots,\mU_N, \calS )\|_F^2\nonumber\\
    &\!\!\!\!\leq\!\!\!\!&\sum_{i=1}^N\frac{1}{\overline{\sigma}_{\text{tk}}^2(\calX^\star)}\|\nabla_{\mU_i}H_1(\mU_1,\dots, \mU_N,\calS)\|_F^2+\|\nabla_{\calS}H_1(\mU_1,\dots,\mU_N, \calS )\|_F^2\nonumber\\
    &\!\!\!\!\leq\!\!\!\!&\frac{9N+4}{4}\|\nabla_{\calX} h(\calX)-\nabla_{\calX} h(\calX^\star)\|_F^2,
\end{eqnarray}
where the first inequality follows from the fact that for any matrix $\mB = \calP_{\text{T}_{L({\mX}_i)} \text{St}}(\mB) + \calP_{\text{T}_{L({\mX}_{i})} \text{St}}^{\perp}(\mB)$ where $\calP_{\text{T}_{L({\mX}_i)} \text{St}}(\mB)$ and $\calP_{\text{T}_{L({\mX}_{i})} \text{St}}^{\perp}(\mB)$ are orthogonal, we have $\|\calP_{\text{T}_{L({\mX}_i)} \text{St}}(\mB)\|_F^2\leq \|\mB\|_F^2$.

\paragraph{Lower bound of the cross term in \eqref{definition of regularity condition for tucker gene loss}} We first rewrite the cross term in \eqref{definition of regularity condition for tucker gene loss} as follows:
\begin{align}
    \label{Lower bound of cross term in the Riemannain Tucker general_ original}
    &\sum_{i=1}^N\<\mU_i-\mU_i^\star\mR_i, \calP_{\text{T}_{\mU_i} \text{St}}(\nabla_{\mU_i}H_1(\mU_1,\dots,\mU_N, \calS ))\>+\<\calS-[\![\calS^\star; \mR_{1}^\top,\dots, \mR_{N}^\top  ]\!], \nabla_{\calS}H_1(\mU_1,\dots,\mU_N, \calS ) \>\nonumber\\
    &=\bigg\<(N+1)\calX-\sum_{i=1}^N\calS\times_1\mU_1\times_2\cdots\times_{i-1}\mU_{i-1}\times_i\mU_i^\star\mR_i\times_{i+1}\mU_{i+1}\times_{i+2}\cdots\times_N\mU_N
    -\calS^\star\times_1\mU_1\mR_{1}^\top\nonumber\\
    &\times_2\cdots\times_N \mU_N\mR_{N}^\top, \nabla_{\calX} h(\calX)-\nabla_{\calX} h(\calX^\star) \>- T_1\nonumber\\
    &=\<\calX-\calX^\star+\calH, \nabla_{\calX} h(\calX)-\nabla_{\calX} h(\calX^\star)\>  - T_1,
\end{align}
where we define
\begin{eqnarray}
    \label{The definition of H_t}
    \mathcal{H}&\!\!\!\!=\!\!\!\!& \calX^\star  -\calS^\star\times_1\mU_1\mR_{1}^\top\times_2{\mU_{2}} \mR_{2}^\top\times_3\cdots\times_N \mU_N\mR_{N}^\top\nonumber\\
    &\!\!\!\!\!\!\!\!&  +  \sum_{i=1}^N  \calS\times_1\mU_1\times_2\cdots\times_{i-1}\mU_{i-1}\times_i( \mU_i -  \mU_i^\star
    \mR_i)\times_{i+1}\mU_{i+1}\times_{i+2}\cdots\times_N\mU_N,
\end{eqnarray}
and
\begin{eqnarray}
    \label{The definition of T1 Tucker}
    T_1=\sum_{i=1}^N\<\mU_i-\mU_i^\star\mR_i, \calP^\perp_{\text{T}_{\mU_i} \text{St}}(\nabla_{\mU_i}H_1(\mU_1,\dots, \mU_N,\calS ))\>.
\end{eqnarray}

To further analyze the lower bound of \eqref{Lower bound of cross term in the Riemannain Tucker general_ original}, we need to derive the upper bounds of \eqref{The definition of H_t} and \eqref{The definition of T1 Tucker}. Specifically, we respectively have
\begin{eqnarray}
    \label{The upper bound of H_t}
    \|\calH\|_F^2&\!\!\!\!=\!\!\!\!&\bigg\|\sum_{i=1}^{N-1}\sum_{j=i+1}^N \calS\times_1\mU_1\times_2\cdots\times_{i-1}\mU_{i-1}\times_i(\mU_i-\mU_i^\star\mR_i)\times_{i+1}\mU_{i+1}^\star\mR_{i+1}\times_{i+2}\nonumber\\
    &\!\!\!\!\!\!\!\!&\cdots \times_{j-1}\mU_{j-1}^\star\mR_{j-1}\times_j ({\mU_j}-\mU_{j}^\star{\mR_j})\times_{j+1}\cdots\times_N\mU_N + \sum_{k=1}^N(\calS-[\![\calS^\star; \mR_{1}^\top,\dots, \mR_{N}^\top  ]\!])\nonumber\\
    &\!\!\!\!\!\!\!\!&\times_1\mU_1\times_2\cdots\times_{k-1}\mU_{k-1}\times_k(\mU_k
    -\mU_k^\star\mR_{k})\times_{k+1}\mU_{k+1}^\star{\mR_{k+1}}\times_{k+2}\cdots\times_{N}\mU_{N}^\star{\mR_N}\bigg\|_F^2\nonumber\\
    &\!\!\!\!\leq\!\!\!\!&\frac{N(N+1)}{2}\bigg(\frac{9}{4}\sum_{i=1}^{N-1}\|\mU_i - \mU_i^\star\mR_i \|_F^2\bigg(\sum_{j=i+1}^{N} \overline{\sigma}_{\text{tk}}^2(\calX^\star)\|{\mU_j} - \mU_j^\star\mR_j \|_F^2 + \|\calS-[\![\calS^\star; \mR_1^\top,\dots, \mR_N^\top  ]\!]\|_F^2\bigg)\nonumber\\
    &\!\!\!\!\!\!\!\!&+\|\mU_N-\mU_N^\star\mR_N\|_F^2\|\calS-[\![\calS^\star; \mR_1^\top,\dots, \mR_N^\top  ]\!]\|_F^2\bigg)\nonumber\\
    &\!\!\!\!\leq\!\!\!\!&\frac{9N(N+1)}{8}\sum_{i=1}^N\|\mU_i-\mU_i^\star\mR_i\|_F^2\text{dist}_\text{tk}^2(\{\mU_i,\calS \}, \{\mU_i^\star,\calS^\star \})\nonumber\\
    &\!\!\!\!\leq\!\!\!\!&\frac{9N(N+1)}{8\overline{\sigma}_{\text{tk}}^2(\calX^\star)}\text{dist}_\text{tk}^4(\{\mU_i,\calS \}, \{\mU_i^\star,\calS^\star \}),
\end{eqnarray}
where \Cref{TRANSFORMATION OF N+1 to N(N+1)} is used in the first line; and following the another expression of orthogonal complement for $\mU_i-\mU_i^\star\mR_i$: $\calP^\perp_{\text{T}_{\mU_i} \text{St}}(\mU_i-\mU_i^\star\mR_i)=\frac{1}{2}\mU_i\bigg((\mU_i-\mU_i^\star\mR_i)^\top\mU_i
    +\mU_i^\top(\mU_i-\mU_i^\star\mR_i)\bigg)=\frac{1}{2}\mU_i(2\mId_{r_i^\text{tk}}-\mR_i^\top{\mU_i^\star}^\top\mU_i-\mU_i^\top\mU_i^\star\mR_i )=\frac{1}{2}\mU_i(\mU_i-\mU_i^\star\mR_i)^\top(\mU_i-\mU_i^\star\mR_i)$,
we have
\begin{eqnarray}
    \label{Upper bound of orthogonoal projection of gradient in the Stiefel general}
    T_1&\!\!\!\!=\!\!\!\!&\sum_{i=1}^N\<\mU_i-\mU_i^\star\mR_i, \calP^\perp_{\text{T}_{\mU_i} \text{St}}(\nabla_{\mU_i}H_1(\mU_1,\dots, \mU_N,\calS ))\>\nonumber\\
    &\!\!\!\!=\!\!\!\!&\sum_{i=1}^N\<\calP^\perp_{\text{T}_{\mU_i} \text{St}}(\mU_i-\mU_i^\star\mR_i),\nabla_{\mU_i}H_1(\mU_1,\dots,\mU_N, \calS )\>\nonumber\\
    &\!\!\!\!=\!\!\!\!&\frac{1}{2}\sum_{i=1}^N\< \mU_i(\mU_i-\mU_i^\star\mR_i)^\top(\mU_i-\mU_i^\star\mR_i),  \nabla_{\mU_i}H_1(\mU_1,\dots, \mU_N, \calS ) \>\nonumber\\
    &\!\!\!\!\leq\!\!\!\!&\frac{1}{2}\sum_{i=1}^N\|\mU_i\| \|(\mU_i-\mU_i^\star\mR_i)^\top(\mU_i-\mU_i^\star\mR_i)\|_F\|\nabla_{\mU_i}H_1(\mU_1,\dots, \mU_N, \calS )\|_F\nonumber\\
    &\!\!\!\!\leq\!\!\!\!&\frac{3\overline{\sigma}_{\text{tk}}^2(\calX^\star)}{4}\sum_{i=1}^N\|\mU_i-\mU_i^\star\mR_i\|_F^2\|\nabla_{\calX} h(\calX)-\nabla_{\calX} h(\calX^\star)\|_F\nonumber\\
    &\!\!\!\!\leq\!\!\!\!&\frac{\beta}{4}\|\nabla_{\calX} h(\calX)-\nabla_{\calX} h(\calX^\star)\|_F^2+\frac{9\overline{\sigma}_{\text{tk}}^2(\calX^\star)N}{16\beta}\sum_{i=1}^N\|\mU_i-\mU_i^\star\mR_i\|_F^4\nonumber\\
    &\!\!\!\!\leq\!\!\!\!&\frac{\beta}{4}\|\nabla_{\calX} h(\calX)-\nabla_{\calX} h(\calX^\star)\|_F^2+\frac{9N}{16\beta\overline{\sigma}_{\text{tk}}^2(\calX^\star)}\text{dist}_\text{tk}^4(\{\mU_i,\calS \}, \{\mU_i^\star,\calS^\star \}),
\end{eqnarray}
where \eqref{Upper bound of squared term in the tensor general Riemannian 1ToN} is utilized in the second inequality.

Now by $\nabla_{\calX} h(\calX^\star) = 0$, we can bound the cross term of \eqref{Lower bound of cross term in the Riemannain Tucker general_ original} as follows:
\begin{align}
    \label{Lower bound of cross term in the Riemannain Tucker general}
    &\sum_{i=1}^N\<\mU_i-\mU_i^\star\mR_i, \calP_{\text{T}_{\mU_i} \text{St}}(\nabla_{\mU_i}H_1(\mU_1,\dots,\mU_N, \calS ))\>+\<\calS-[\![\calS^\star; \mR_{1}^\top,\dots, \mR_{N}^\top  ]\!], \nabla_{\calS}H_1(\mU_1,\dots,\mU_N, \calS ) \>\nonumber\\
    &\geq\alpha\|\calX-\calX^\star\|_F^2-\frac{1}{2\beta}\|\calH\|_F^2+\frac{\beta}{4}\|\nabla_{\calX} h(\calX)-\nabla_{\calX} h(\calX^\star)\|_F^2-\frac{9N}{16\beta\overline{\sigma}_{\text{tk}}^2(\calX^\star)}
    \text{dist}_\text{tk}^4(\{\mU_i,\calS \}, \{\mU_i^\star,\calS^\star \})\nonumber\\
    &\geq\alpha\|\calX-\calX^\star\|_F^2-\frac{9N(N+2)}{16\beta\overline{\sigma}_{\text{tk}}^2(\calX^\star)}
    \text{dist}_\text{tk}^4(\{\mU_i,\calS \}, \{\mU_i^\star,\calS^\star \})+\frac{\beta}{4}\|\nabla_{\calX} h(\calX)-\nabla_{\calX} h(\calX^\star)\|_F^2\nonumber\\
    &\geq\frac{\alpha}{(9N^2+14N+1)\kappa^2_\text{tk}(\calX^\star)}\text{dist}_\text{tk}^2(\{\mU_i,\calS \}, \{\mU_i^\star,\calS^\star \})-\frac{9N(N+2)}{16\beta\overline{\sigma}_{\text{tk}}^2(\calX^\star)}
    \text{dist}_\text{tk}^4(\{\mU_i,\calS \}, \{\mU_i^\star,\calS^\star \})\nonumber\\
    &+\frac{\beta}{4}\|\nabla_{\calX} h(\calX)-\nabla_{\calX} h(\calX^\star)\|_F^2\nonumber\\
    &\geq\frac{\alpha}{2(9N^2+14N+1)\kappa^2_\text{tk}(\calX^\star)}\text{dist}_\text{tk}^2(\{\mU_i,\calS \}, \{\mU_i^\star,\calS^\star \})+\frac{\beta}{4}\|\nabla_{\calX} h(\calX)-\nabla_{\calX} h(\calX^\star)\|_F^2,
\end{align}
where the first and second inequalities respectively follow \eqref{Upper bound of orthogonoal projection of gradient in the Stiefel general}, the RCG condition, and \eqref{The upper bound of H_t}. In addition, we make use of \Cref{Lower bound of two distance in the tucker tensor_1} and $\text{dist}_\text{tk}^2(\{\mU_i,\calS \}, \{\mU_i^\star,\calS^\star \})\leq \frac{8\alpha\beta\underline{\sigma}_{\text{tk}}^2(\calX^\star)}{9N(N+2)(9N^2+14N+1)}$ in the penultimate and last lines.

\paragraph{Contraction} By \eqref{Upper bound of squared term in the tensor general Riemannian Conclusion} and \eqref{Lower bound of cross term in the Riemannain Tucker general}, we can obtain
\begin{align}
    \label{regularity condition for tucker gene loss in the appendix}
    &\hspace{0.4cm} \sum_{i=1}^N\bigg\<\mU_i-\mU_i^\star\mR_i, \calP_{\text{T}_{\mU_i} \text{St}}(\nabla_{\mU_i}H_1(\mU_1,\dots,\mU_N, \calS ))\bigg\>+\bigg\<\calS-[\![\calS^\star; \mR_{1}^\top,\dots, \mR_{N}^\top  ]\!], \nabla_{\calS}H_1(\mU_1,\dots,\mU_N, \calS ) \bigg\>\nonumber\\
    &\geq \frac{\alpha}{2(9N^2+14N+1)\kappa^2_\text{tk}(\calX^\star)}\text{dist}_\text{tk}^2(\{\mU_i,\calS \}, \{\mU_i^\star,\calS^\star \})
    +\frac{\beta}{9N+4}\bigg(\frac{1}{\overline{\sigma}_{\text{tk}}^2(\calX^\star)}\sum_{i=1}^N\|\calP_{\text{T}_{\mU_i} \text{St}}(\nabla_{\mU_i}H_1(\mU_1,\dots, \mU_N,\calS ))\|_F^2\nonumber\\
    &\hspace{0.4cm}+\|\nabla_{\calS}H_1(\mU_1,\dots,\mU_N, \calS )\|_F^2  \bigg).
\end{align}

\end{proof}

\section{Riemannian regularity condition for TT format}
\label{proof of regularity condition of general loss in TT format}

In this section, we initially articulate the loss function for TT format as follows:
\begin{align}
    \label{RIEMANNIAN_LOSS_FUNCTION_1 general in intro}
   \begin{split}  \text{TT format:}  \ \ \
\min_{\mX_i\in\R^{r_{i-1}^\text{tt}\times d_i \times r_i^\text{tt}}, \atop i\in[N]}\!\!  H_2(\mX_1,\dots, \mX_N) & = h( [\mX_1,\dots, \mX_N]),\\
\st \ L^\top(\mX_i)L(\mX_i)&=\mId_{r_i^\text{tt}}, \ i\in [N-1].
    \end{split}
\end{align}

Then we define the following distance metric:
\begin{eqnarray}
\label{BALANCED NEW DISTANCE BETWEEN TWO TENSORS appe}
    \text{dist}^2_\text{tt}(\{\mX_i  \},\{ \mX_i^\star \})=\min_{\mR_i\in\O^{r_i^\text{tt}\times r_i^\text{tt}}, \atop i \in [N-1]}\sum_{i=1}^{N-1} \ol\sigma_\text{tt}^2(\calX^\star)\|L({\mX}_i)-L_{\mR}({\mX}_i^\star)\|_F^2 + \|L({\mX}_N)-L_{\mR}({\mX}_N^\star)\|_2^2.
\end{eqnarray}
Now, we have the Riemannian regularity condition as follows:
\begin{lemma}
\label{regularity condition of general loss in the TT format}
Let a rank-$(r_1^\text{tt},\dots, r_{N-1}^\text{tt})$ TT format $\calX^\star = [\mX_1^\star,\dots,\mX_N^\star ]$ be a target tensor of a loss function $h$ in \eqref{RIEMANNIAN_LOSS_FUNCTION_1 general in intro} which satisfies the RCG condition with   the set $\calC_\text{tt}(c^\text{tt}a_1^\text{tt})$ as follows:
\begin{eqnarray}
    \label{the defitinition of the intial set of generel loss in TT}
    \calC_\text{tt}(a_1^\text{tt}) : =\bigg\{\calX = [\mX_1,\dots,\mX_N ]: \|\calX - \calX^\star\|_F^2\leq c^\text{tt}a_1^\text{tt} \bigg\},
\end{eqnarray}
where $a_1^\text{tt} = \frac{\alpha\beta\underline{\sigma}^2_\text{tt}(\calX^\star)}{9(N^2-1)(N+1+\sum_{i=2}^{N-1}r_i^\text{tt})}$ and $c^\text{tt} = \frac{1}{8(N+1+\sum_{i=2}^{N-1}r_i^\text{tt})\kappa_\text{tt}^2(\calX^\star)}$. The loss function $H_2$ in \eqref{RIEMANNIAN_LOSS_FUNCTION_1 general in intro} satisfies a Riemannian regularity condition, denoted by $\text{RRC}(a_2^\text{tt},a_3^\text{tt})$, if for all $\{\mX_i  \}\in\{\{\mX_i  \}: \text{dist}_\text{tt}^2(\{\mX_i  \}, \{\mX_i^\star\})\leq a_1^\text{tt} \}$,  the following inequality holds:
\begin{align}
    \label{the defitinition of regularity condition for TT gene loss}
    &\sum_{i=1}^{N} \bigg\< L(\mX_i)-L_{\mR}(\mX_i^\star),\calP_{\text{T}_{L(\mX_i)} \text{St}}\bigg(\nabla_{L(\mX_{i})}H_2(\mX_1, \dots, \mX_N)\bigg)\bigg\> \geq a_2^\text{tt}\text{dist}_\text{tt}^2(\{\mX_i  \}, \{\mX_i^\star\})\nonumber\\
    & + a_3^\text{tt}\bigg( \frac{1}{\ol\sigma_\text{tt}^2(\calX^\star)}\sum_{i=1}^{N-1}\|\calP_{\text{T}_{L(\mX_i)} \text{St}}(\nabla_{L(\mX_{i})}H_2(\mX_1, \dots, \mX_N))\|_F^2+\|\nabla_{L(\mX_N)}H_2(\mX_1, \dots, \mX_N) \|_2^2  \bigg),
\end{align}
where $a_2^\text{tt} = \frac{\alpha}{16(N+1+\sum_{i=2}^{N-1}r_i^\text{tt})\kappa_\text{tt}^2(\calX^\star)}$, $a_3^\text{tt} = \frac{\beta}{9N-5}$, and $(\mR_1,\dots, \mR_{N}) = \argmin_{\wt\mR_i\in\O^{r_i^\text{tt}\times r_i^\text{tt}}, i\in[{N}]}\text{dist}_\text{tt}^2(\{\mX_i  \}, \{\mX_i^\star\})$. To simplify the expression, we define the identity operator $\calP_{\text{T}_{L({\mX}_N)} \text{St}}=\calI$ such that $\calP_{\text{T}_{L({\mX}_N)} \text{St}}(\nabla_{L({\mX}_{N})}H_2(\mX_1, \dots, \mX_N)) = \nabla_{L({\mX}_{N})}H_2(\mX_1, \dots, \mX_N)$.
\end{lemma}

\begin{proof}

Before proving \Cref{regularity condition of general loss in the TT format}, we first provide one useful property. Specifically, we have
\begin{eqnarray}
\label{upper bound TT spctral norm unified}
    \sigma_1^2({\calX}^{\<i \>}) &\!\!\!\!= \!\!\!\!&  \|{\calX}^{\geq i+1} \|^2 \leq \min_{\mR_i\in\O^{r_i\times r_i}}2\|\mR_{i}^\top{\calX^{\star}}^{\geq i+1} \|^2 + 2\|{\calX}^{\geq i+1} - \mR_{i}^\top{\calX^{\star}}^{\geq i+1} \|^2\nonumber\\
    &\!\!\!\!\leq \!\!\!\!&2\ol{\sigma}^2_\text{tt}(\calX^\star) + \min_{\mR_i\in\O^{r_i\times r_i}}2 \|{\calX}^{\<i \>} - {\calX^\star}^{\<i \>} + {\calX^\star}^{\<i \>} - {\calX}^{\leq i}\mR_{i}^\top{\calX^{\star}}^{\geq i+1}  \|^2\nonumber\\
    &\!\!\!\!\leq \!\!\!\!&  2\ol{\sigma}^2_\text{tt}(\calX^\star) + 4\|\calX - \calX^\star \|_F^2 + \min_{\mR_i\in\O^{r_i\times r_i}}4 \|\mR_{i}^\top{\calX^{\star}}^{\geq i+1} \|^2 \|{\calX}^{\leq i} - {\calX^\star}^{\leq i} \mR_i \|_F^2\nonumber\\
    &\!\!\!\!\leq \!\!\!\!& 2\ol{\sigma}^2_\text{tt}(\calX^\star) + \bigg(4 + \frac{16\ol{\sigma}^2_\text{tt}(\calX^\star)}{\underline{\sigma}^2_\text{tt}(\calX^\star)}\bigg)\|\calX-\calX^\star\|_F^2\nonumber\\
    &\!\!\!\!\leq \!\!\!\!& 2\ol{\sigma}^2_\text{tt}(\calX^\star) + \frac{45N\ol{\sigma}^2_\text{tt}(\calX^\star)}{\underline{\sigma}^2_\text{tt}(\calX^\star)} \text{dist}^2_\text{tt}(\{\mX_i \},\{ \mX_i^\star \})\nonumber\\
    &\!\!\!\!\leq \!\!\!\!& 2\ol{\sigma}^2_\text{tt}(\calX^\star) + \frac{45N\alpha\beta\ol{\sigma}^2_\text{tt}(\calX^\star)}{9(N^2-1)(N+1+\sum_{i=2}^{N-1}r_i^\text{tt})}\nonumber\\
    &\!\!\!\!\leq \!\!\!\!& \frac{9\ol{\sigma}^2_\text{tt}(\calX^\star)}{4}, i\in[N-1],
\end{eqnarray}
where the fourth inequality follows \Cref{left ortho upper bound} and the last line uses $\alpha\beta\leq \frac{1}{4}$ \cite{Han20} and $N\geq 3$. Note that  $\ol{\sigma}^2_\text{tt}(\calX) = \max_{i=1}^{N-1}\sigma_1^2({\calX}^{\<i \>})\leq \frac{9\ol{\sigma}^2_\text{tt}(\calX^\star)}{4}$.

Note that the gradient is defined as
\begin{eqnarray}
    \label{the gradient of general loss in the TT format}
    \nabla_{L(\mX_{i})}H_2(\mX_1, \dots, \mX_N) = \begin{bmatrix}\nabla_{\mX_i(1)} H_2(\mX_1, \dots, \mX_N)\\ \vdots \\ \nabla_{\mX_i(d_i)} H_2(\mX_1, \dots, \mX_N)  \end{bmatrix}, \ \ i\in[N],
\end{eqnarray}
where following \cite[Lemma 2.1]{Han20}, the gradient with respect to each factor $\mX_{i}(s_i)$ can be computed as
\begin{align*}
\nabla_{\mX_i(s_i)}H_2(\mX_1, \dots, \mX_N)
=\hspace{-1cm} \sum_{s_1,\ldots,s_{i-1},s_{i+1},\ldots,s_N } \hspace{-1cm} \Big(& \nabla_{\calX(s_1,\dots,s_N)}h(\calX)\mX_{i-1}^\top(s_{i-1})\cdots\mX_{1}^\top(s_{1})\cdot\\
& \mX_{N}^\top(s_{N})\cdots\mX_{i+1}^\top(s_{i+1}) \Big).
\end{align*}

\paragraph{Upper bound of the squared term in \eqref{the defitinition of regularity condition for TT gene loss}}
We first define three matrices for $i\in[N]$ as follows:
\begin{eqnarray}
    \label{The definition of D1 general loss}
    \mD_1(i) &\!\!\!\!=\!\!\!\!& \begin{bmatrix} \mX_{i-1}^\top(1)\!\cdots\!\mX_{1}^\top(1) \ \ \ \   \cdots \ \ \ \   \mX_{i-1}^\top(d_{i-1})\!\cdots\!\mX_{1}^\top(d_{1})   \end{bmatrix}\nonumber\\
    &\!\!\!\!=\!\!\!\!&L^\top(\mX_{i-1})\ol \otimes \cdots \ol \otimes L^\top(\mX_1)\in\R^{r_i^\text{tt}\times(d_1\cdots d_{i-1})},\\
    \mD_2(i) &\!\!\!\!=\!\!\!\!& \begin{bmatrix} \mX_{N}^\top(1)\cdots \mX_{i+1}^\top(1)\\ \vdots \\ \mX_{N}^\top(d_N)\cdots \mX_{i+1}^\top(d_{i+1}) \end{bmatrix}\in\R^{(d_{i+1}\cdots d_N )\times r_i^\text{tt} },
\end{eqnarray}
where we note that $\mD_1(1) = 1$ and $\mD_2(N) = 1$.
Moreover, for each $s_i\in [d_i]$, we define the matrix $\mE(s_i)\in\R^{(d_1\cdots d_{i-1})\times (d_{i+1}\cdots d_N)}$ whose $(s_1\cdots s_{i-1}, s_{i+1}\cdots s_N)$-th element  is given by
\begin{eqnarray}
    \label{Each element of E_si general loss}
    \mE(s_i)(s_1\cdots s_{i-1}, s_{i+1}\cdots s_N) =  \nabla_{\calX(s_1,\dots,s_N)}h(\calX) - \nabla_{\calX(s_1,\dots,s_N)}h(\calX^\star),
\end{eqnarray}
where we note that $\nabla_{\calX}h(\calX^\star) = 0$.

Based on the aforementioned notations, we can derive
\begin{eqnarray}
    \label{PROJECTED GRADIENT DESCENT SQUARED TERM 1 to N general loss}
    \bigg\| \nabla_{L({\mX}_{i})}H_2(\mX_1, \dots, \mX_N)\bigg\|_F^2&\!\!\!\!=\!\!\!\!& \sum_{s_i=1}^{d_i}\bigg\| \nabla_{\mX_{i}(s_i)} H_2(\mX_1, \dots, \mX_N)\bigg\|_F^2 =\sum_{s_i=1}^{d_i}\|\mD_1(i) \mE(s_i) \mD_2(i)   \|_F^2\nonumber\\
    &\!\!\!\!\leq\!\!\!\!& \sum_{s_i=1}^{d_i}\|L^\top(\mX_{i-1})\ol \otimes \cdots \ol \otimes L^\top(\mX_1)\|^2\| \mD_2(i)\|^2 \|\mE(s_i)\|_F^2\nonumber\\
    &\!\!\!\!\leq\!\!\!\!&\|L(\mX_1)\|^2\cdots\|L(\mX_{i-1})\|^2 \|{\calX}^{\geq i+1} \|^2  \|\nabla_{\calX}h(\calX)-\nabla_{\calX}h(\calX^\star)\|_F^2\nonumber\\
    &\!\!\!\!\leq\!\!\!\!&\begin{cases}
    \frac{9\ol\sigma_\text{tt}^2(\calX^\star)}{4}\|\nabla_{\calX}h(\calX)-\nabla_{\calX}h(\calX^\star)\|_F^2, & i\in[N-1],\\
    \|\nabla_{\calX}h(\calX)-\nabla_{\calX}h(\calX^\star)\|_F^2, & i = N,
  \end{cases}
\end{eqnarray}
where we use \eqref{KRONECKER PRODUCT VECTORIZATION11 - 2}, $\|\mD_2(i)\| =  \|{\calX}^{\geq i+1} \| $, and $\sum_{s_i=1}^{d_i}\|\mE(s_i)\|_F^2 = \|\nabla_{\calX}h(\calX)-\nabla_{\calX}h(\calX^\star)\|_F^2$ in the second inequality. In addition, the third inequality follows $\|{\calX}^{\geq i+1} \| = \sigma_1({\calX}^{\<i \>})\leq \frac{3 \ol{\sigma}_\text{tt}(\calX^\star)}{2}$.

Using \eqref{PROJECTED GRADIENT DESCENT SQUARED TERM 1 to N general loss}, we can easily obtain
\begin{eqnarray}
    \label{RIEMANNIAN FACTORIZATION SQUARED TERM UPPER BOUND general loss}
    &\!\!\!\!\!\!\!\!&\frac{1}{\ol{\sigma}_\text{tt}^2(\calX^\star)}\sum_{i=1}^{N-1}\bigg\|\mathcal{P}_{{\rm T}_{L({\mX}_i)} {\rm St}}\bigg(\nabla_{L({\mX}_{i})}H_2(\mX_1, \dots, \mX_N)\bigg)\bigg\|_F^2+\bigg\|\nabla_{L({\mX}_{N})}H_2(\mX_1, \dots, \mX_N) \bigg\|_2^2\nonumber\\
    &\!\!\!\!\leq\!\!\!\!&\frac{1}{\ol{\sigma}_\text{tt}^2(\calX^\star)}\sum_{i=1}^{N-1}\bigg\|\nabla_{L({\mX}_{i})}H_2(\mX_1, \dots, \mX_N)\bigg\|_F^2+\bigg\|\nabla_{L({\mX}_{N})}H_2(\mX_1, \dots, \mX_N) \bigg\|_2^2\nonumber\\
    &\!\!\!\!\leq\!\!\!\!&\frac{9N-5}{4}\|\nabla_{\calX}h(\calX)-\nabla_{\calX}h(\calX^\star)\|_F^2,
\end{eqnarray}
where the first inequality follows from the fact that for any matrix $\mB = \calP_{\text{T}_{L({\mX}_i)} \text{St}}(\mB) + \calP_{\text{T}_{L({\mX}_{i})} \text{St}}^{\perp}(\mB)$ where $\calP_{\text{T}_{L({\mX}_i)} \text{St}}(\mB)$ and $\calP_{\text{T}_{L({\mX}_{i})} \text{St}}^{\perp}(\mB)$ are orthogonal,  we have $\|\calP_{\text{T}_{L({\mX}_i)} \text{St}}(\mB)\|_F^2\leq \|\mB\|_F^2$.

\paragraph{Lower bound of the cross term in \eqref{the defitinition of regularity condition for TT gene loss}}

First, we rewrite the cross term in \eqref{the defitinition of regularity condition for TT gene loss} as follows:
\begin{eqnarray}
    \label{RIEMANNIAN general CROSS TERM LOWER BOUND original}
    &\!\!\!\!\!\!\!\!&\sum_{i=1}^{N} \bigg\< L(\mX_i)-L_{\mR}(\mX_i^\star),\calP_{\text{T}_{L(\mX_i)} \text{St}}\bigg(\nabla_{L(\mX_{i})}H_2(\mX_1, \dots, \mX_N)\bigg)\bigg\>\nonumber\\
    &\!\!\!\!=\!\!\!\!&\sum_{i=1}^{N} \bigg\< L(\mX_i)-L_{\mR}(\mX_i^\star), \nabla_{L(\mX_{i})}H_2(\mX_1, \dots, \mX_N)\bigg\> - T_2\nonumber\\
    &\!\!\!\!=\!\!\!\!&\bigg\<{\rm vec}(\nabla_{\calX}h(\calX)-\nabla_{\calX}h(\calX^\star)),L(\mX_1) \overline{\otimes} \cdots \overline{\otimes} L(\mX_N)-L_{\mR}(\mX_1^\star) \overline{\otimes} \cdots \overline{\otimes} L_{\mR}(\mX_N^\star)+\vh\bigg\>- T_2,
\end{eqnarray}
where we respectively define
\begin{eqnarray}
    \label{H_T IN THE CROSS TERM}
    \vh&\!\!\!\!=\!\!\!\!&L_{\mR}(\mX_1^\star)\ol \otimes  \cdots \ol \otimes L_{\mR}({\mX}_N^\star) - L(\mX_1) \ol \otimes \cdots \ol \otimes L(\mX_{N-1}) \ol \otimes L_{\mR}(\mX_{N}^\star)\nonumber\\
    &\!\!\!\!\!\!\!\!& +\sum_{i=1}^{N-1}L(\mX_1)\ol \otimes\cdots\ol \otimes L(\mX_{i-1})\ol \otimes ( L({\mX}_i) -L_{\mR}({\mX}_i^\star) ) \ol \otimes L(\mX_{i+1})\ol \otimes \cdots \ol \otimes L(\mX_N),
\end{eqnarray}
and
    \begin{eqnarray}
        \label{PROJECTION ORTHOGONAL IN RIEMANNIAN UPPER BOUND general original}
        T_2&\!\!\!\! = \!\!\!\!&\sum_{i=1}^{N-1}\bigg\<L(\mX_i)-L_{\mR}(\mX_i^\star), \calP^{\perp}_{\text{T}_{L(\mX_i)} \text{St}}(\nabla_{L(\mX_{i})}H_2(\mX_1, \dots, \mX_N)) \bigg\>\nonumber\\
        & \!\!\!\!= \!\!\!\!&\sum_{i=1}^{N-1}\bigg\<\calP^{\perp}_{\text{T}_{L(\mX_i)} \text{St}}(L(\mX_i)-L_{\mR}(\mX_i^\star)), \nabla_{L(\mX_{i})}H_2(\mX_1, \dots, \mX_N) \bigg\>.
    \end{eqnarray}

Before deriving the lower bound of \eqref{RIEMANNIAN general CROSS TERM LOWER BOUND original}, we need to obtain the upper bounds of \eqref{H_T IN THE CROSS TERM} and \eqref{PROJECTION ORTHOGONAL IN RIEMANNIAN UPPER BOUND general original}. Specifically, according to \cite[eq. (82)]{qin2024guaranteed} we have
\begin{eqnarray}
    \label{H_T IN THE CROSS TERM UPPER BOUND}
    \|\vh\|_2^2\leq\frac{9N(N-1)}{8\ol{\sigma}_\text{tt}^2(\calX^\star)}\text{dist}_\text{tt}^4(\{\mX_i  \}, \{\mX_i^\star\}).
\end{eqnarray}
In addition, we can derive
\begin{eqnarray}
        \label{PROJECTION ORTHOGONAL IN RIEMANNIAN UPPER BOUND general}
        T_2&\!\!\!\!\leq\!\!\!\!&\sum_{i=1}^{N-1}\frac{1}{2}\|L(\mX_i)\|\|L(\mX_i)-L_{\mR}(\mX_i^\star)\|_F^2\bigg|\bigg|\nabla_{L(\mX_{i})}H_2(\mX_1, \dots, \mX_N)\bigg|\bigg|_F\nonumber\\
        &\!\!\!\!\leq\!\!\!\!&\sum_{i=1}^{N-1}\frac{3\ol{\sigma}_\text{tt}(\calX^\star)}{4}\|L(\mX_i)-L_{\mR}(\mX_i^\star)\|_F^2\|\nabla_{\calX}h(\calX)-\nabla_{\calX}
        h(\calX^\star)\|_F\nonumber\\
        &\!\!\!\!\leq\!\!\!\!&\frac{\beta}{4}\|\nabla_{\calX}h(\calX)-\nabla_{\calX}h(\calX^\star)\|_F^2+\frac{9(N-1)\ol{\sigma}_\text{tt}^2(\calX^\star)}{16\beta}
        \sum_{i=1}^{N-1}\|L(\mX_i)
        -L_{\mR}(\mX_i^\star)\|_F^4\nonumber\\
        &\!\!\!\!\leq\!\!\!\!&\frac{\beta}{4}\|\nabla_{\calX}h(\calX)-\nabla_{\calX}h(\calX^\star)\|_F^2+\frac{9(N-1)}{16\beta\ol{\sigma}_\text{tt}^2(\calX^\star)}
        \text{dist}_\text{tt}^4(\{\mX_i  \}, \{\mX_i^\star\}),
\end{eqnarray}
    where the second inequality follows from \eqref{PROJECTED GRADIENT DESCENT SQUARED TERM 1 to N general loss} and $\calP^{\perp}_{\text{T}_{L(\mX_i)} \text{St}}(\cdot)$ is defined as
\begin{eqnarray}
        \label{Projection orthogonal in the Riemannian gradient descent general}
        &\!\!\!\!\!\!\!\!&\calP^{\perp}_{\text{T}_{L(\mX_i)} \text{St}}(L(\mX_i)-L_{\mR}(\mX_i^\star))\nonumber\\
        &\!\!\!\!=\!\!\!\!&\frac{1}{2}L(\mX_i)((L(\mX_i)-L_{\mR}(\mX_i^\star))^\top L(\mX_i)+L^\top(\mX_i)(L(\mX_i)-L_{\mR}(\mX_i^\star)))\nonumber\\
        &\!\!\!\!=\!\!\!\!&\frac{1}{2}L(\mX_i)(2\mId_{r_i^\text{tt}}-L_{\mR}^\top(\mX_i^\star)L(\mX_i)-L^\top(\mX_i)L_{\mR}(\mX_i^\star))\nonumber\\
        &\!\!\!\!=\!\!\!\!&\frac{1}{2}L(\mX_i)((L(\mX_i)-L_{\mR}(\mX_i^\star))^\top(L(\mX_i)-L_{\mR}(\mX_i^\star))).
\end{eqnarray}

Now, \eqref{RIEMANNIAN general CROSS TERM LOWER BOUND original} can be further analyzed as
\begin{eqnarray}
    \label{RIEMANNIAN general CROSS TERM LOWER BOUND}
    &\!\!\!\!\!\!\!\!&\sum_{i=1}^{N} \bigg\< L(\mX_i)-L_{\mR}(\mX_i^\star),\calP_{\text{T}_{L(\mX_i)} \text{St}}\bigg(\nabla_{L(\mX_{i})}H_2(\mX_1, \dots, \mX_N)\bigg)\bigg\>\nonumber\\
    &\!\!\!\!\geq\!\!\!\!&\alpha\|\calX-\calX^\star\|_F^2-\frac{1}{2\beta}\|\vh\|_F^2+\frac{\beta}{4}\|\nabla_{\calX}h(\calX)-\nabla_{\calX}h(\calX^\star)\|_F^2
    -\frac{9(N-1)}{16\beta\ol{\sigma}_\text{tt}^2(\calX^\star)}\text{dist}_\text{tt}^4(\{\mX_i  \}, \{\mX_i^\star\})\nonumber\\
    &\!\!\!\!\geq\!\!\!\!&\alpha\|\calX-\calX^\star\|_F^2-\frac{9(N^2-1)}{16\beta\ol{\sigma}_\text{tt}^2(\calX^\star)}\text{dist}_\text{tt}^4(\{\mX_i  \}, \{\mX_i^\star\})
    +\frac{\beta}{4}\|\nabla_{\calX}h(\calX)-\nabla_{\calX}h(\calX^\star)\|_F^2\nonumber\\
    &\!\!\!\!\geq\!\!\!\!&\frac{\alpha}{16(N+1+\sum_{i=2}^{N-1}r_i^\text{tt})\kappa_\text{tt}^2(\calX^\star)}\text{dist}_\text{tt}^2(\{\mX_i  \}, \{\mX_i^\star\})
    +\frac{\beta}{4}\|\nabla_{\calX}h(\calX)-\nabla_{\calX}h(\calX^\star)\|_F^2,
\end{eqnarray}
where we employ \eqref{PROJECTION ORTHOGONAL IN RIEMANNIAN UPPER BOUND general} and the RCG condition in the first inequality, followed by the utilization of \eqref{H_T IN THE CROSS TERM UPPER BOUND} in the second inequality, and conclude by utilizing \Cref{LOWER BOUND OF TWO DISTANCES} alongside the relation $\text{dist}_\text{tt}^2(\{\mX_i  \}, \{\mX_i^\star\})\leq \frac{\alpha\beta\underline{\sigma}^2_\text{tt}(\calX^\star)}{9(N^2-1)(N+1+\sum_{i=2}^{N-1}r_i^\text{tt})}$ in the final step.

\paragraph{Contraction} Combining \eqref{RIEMANNIAN FACTORIZATION SQUARED TERM UPPER BOUND general loss} and \eqref{RIEMANNIAN general CROSS TERM LOWER BOUND}, we can derive
\begin{eqnarray}
    \label{regularity condition for TT gene loss in appendiex}
    &\!\!\!\!\!\!\!\!&\sum_{i=1}^{N} \bigg\< L(\mX_i)-L_{\mR}(\mX_i^\star),\calP_{\text{T}_{L(\mX_i)} \text{St}}\bigg(\nabla_{L(\mX_{i})}H_2(\mX_1, \dots, \mX_N)\bigg)\bigg\>\nonumber\\
    &\!\!\!\!\geq\!\!\!\!& \frac{\alpha}{16(N+1+\sum_{i=2}^{N-1}r_i^\text{tt})\kappa_\text{tt}^2(\calX^\star)}\text{dist}_\text{tt}^2(\{\mX_i  \}, \{\mX_i^\star\}) + \frac{\beta}{9N-5}\bigg(\bigg\|\nabla_{L(\mX_N)}H_2(\mX_1, \dots, \mX_N) \bigg\|_2^2\nonumber\\
    &\!\!\!\!\!\!\!\!&+ \frac{1}{\ol{\sigma}_\text{tt}^2(\calX^\star)}\sum_{i=1}^{N-1}\bigg\|\calP_{\text{T}_{L(\mX_i)} \text{St}}\bigg(\nabla_{L(\mX_{i})}H_2(\mX_1, \dots, \mX_N)\bigg)\bigg\|_F^2 \bigg).
\end{eqnarray}

\end{proof}

\section{Proof of \eqref{RCG conclusion for Tu and TT}}
\label{Proof of RCG in the tensor sensing}

As an immediate consequence of the RIP, the inner product between two low-rank TT formats is also nearly preserved if $\calA$ satisfies the RIP.
\begin{lemma} (\cite{Rauhut17,Han20,qin2024guaranteed})
\label{RIP of tensor sensing another_3}
Suppose that $\calA$ obeys the $2\overline{r}$-RIP with a constant $\delta_{2\overline{r}}$ where $\delta_{2\overline{r}} =     \begin{cases}
    \delta_{2\overline{r}^\text{tk}}, &  \text{Tucker} \\
    \delta_{2\overline{r}^\text{tt}}, &  \text{TT}
    \end{cases}$. Then for any tensor $\calX_1\in\R^{d_{1}\times  \cdots \times  d_{N}}$ and $\calX_2\in\R^{d_{1}\times  \cdots \times d_{N}}$, one has
\begin{eqnarray}
    \label{RIP of tensor sensing another_4}
    \bigg|\frac{1}{m}\<\mathcal{A}(\calX_1),\mathcal{A}(\calX_2)\>-\<\calX_1,\calX_2\>\bigg|\leq \delta_{2\overline{r}}\|\calX_1\|_F\|\calX_2\|_F.
\end{eqnarray}
\end{lemma}

Considering that $\nabla G(\calX) = \frac{1}{m}\calA^*(\calA(\calX - \calX^\star))$ where $\mathcal{A}^*$ is the adjoint operator of $\mathcal{A}$ and is defined as $\mathcal{A}^*(\calX)=\sum_{k=1}^m x_k\calA_k$, and leveraging {Definition} \ref{RIP condition fro the Tucker sensing Lemma}, we have
\begin{eqnarray}
    \label{RCG for G}
    \<\nabla G(\calX)  - \nabla G(\calX^\star), \calX - \calX^\star \> = \frac{1}{m}\|\calA(\calX - \calX^\star)\|_2^2 \geq (1 - \delta_{2\ol r})\|\calX - \calX^\star\|_F^2.
\end{eqnarray}

Then, we define the restricted Frobenius norm for any tensor $\calB\in\R^{d_1\times\cdots\times d_N}$ as follows:
\begin{eqnarray}
    \label{Restricted F norm for Tu and TT}
    \|\calB\|_{F,\ol r} = \begin{cases}
    \|\calB\|_{F,\ol r^\text{tk}} = \max_{\calH\in\R^{d_1\times\cdots\times d_N}, \|\calH\|_F= 1, \atop \text{rank}(\calH)=(r_1^\text{tk},\dots,r_{N}^\text{tk}) } |\<\calB,  \calH \>|, &  \text{Tucker}, \\
    \|\calB\|_{F,\ol r^\text{tt}} = \max_{\calH\in\R^{d_1\times\cdots\times d_N}, \|\calH\|_F= 1, \atop \text{rank}(\calH)=(r_1^\text{tt},\dots,r_{N-1}^\text{tt}) } |\<\calB,  \calH \>|, &  \text{TT}.
    \end{cases}
\end{eqnarray}

Based on \eqref{Restricted F norm for Tu and TT}, we can further derive
\begin{eqnarray}
    \label{lower bound of F norm for Tu and TT}
    \|\calX - \calX^\star\|_F &\!\!\!\!=\!\!\!\!& \|\calX - \calX^\star\|_{F,2 \ol r}\nonumber\\
    &\!\!\!\!=\!\!\!\!& \begin{cases}
    \max_{\calH\in\R^{d_1\times\cdots\times d_N}, \|\calH\|_F= 1, \atop \text{rank}(\calH)=(2r_1^\text{tk},\dots,2r_{N}^\text{tk}) } |\<\calX - \calX^\star,  \calH \>| \\
     \max_{\calH\in\R^{d_1\times\cdots\times d_N}, \|\calH\|_F= 1, \atop \text{rank}(\calH)=(2r_1^\text{tt},\dots,2r_{N-1}^\text{tt}) } |\<\calX - \calX^\star,  \calH \>|
    \end{cases}\nonumber\\
    &\!\!\!\!\geq \!\!\!\!& \begin{cases}
    \max_{\calH\in\R^{d_1\times\cdots\times d_N}, \|\calH\|_F= 1, \atop \text{rank}(\calH)=(2r_1^\text{tk},\dots,2r_{N}^\text{tk}) } \frac{1}{m}\<\mathcal{A}(\calX - \calX^\star),\mathcal{A}(\calH)\> - \delta_{4\ol r^\text{tk}}\|\calX - \calX^\star\|_F\|\calH\|_F \\
     \max_{\calH\in\R^{d_1\times\cdots\times d_N}, \|\calH\|_F= 1, \atop \text{rank}(\calH)=(2r_1^\text{tt},\dots,2r_{N-1}^\text{tt}) } \frac{1}{m}\<\mathcal{A}(\calX - \calX^\star),\mathcal{A}(\calH)\> - \delta_{4\ol r^\text{tt}}\|\calX - \calX^\star\|_F\|\calH\|_F
    \end{cases}\nonumber\\
    &\!\!\!\!= \!\!\!\!& \|\nabla G(\calX) - \nabla G(\calX^\star)\|_{F,2\ol r} - \delta_{4\ol r}\|\calX - \calX^\star\|_F\nonumber\\
    &\!\!\!\!= \!\!\!\!& \|\nabla G(\calX) - \nabla G(\calX^\star)\|_{F} - \delta_{4\ol r}\|\calX - \calX^\star\|_F,
\end{eqnarray}
where the first inequality follows \Cref{RIP of tensor sensing another_3}. We then rewrite \eqref{lower bound of F norm for Tu and TT} as
\begin{eqnarray}
    \label{lower bound of F norm for Tu and TT 1}
    (1 + \delta_{4 \ol r})\|\calX - \calX^\star\|_F \geq  \|\nabla G(\calX) - \nabla G(\calX^\star)\|_{F}.
\end{eqnarray}

Combining \eqref{RCG for G} and \eqref{lower bound of F norm for Tu and TT 1}, we can obtain \eqref{RCG conclusion for Tu and TT}.

\section{Upper bound of spectral initialization for Tucker tensor}
\label{Proof of Tucker initialization}
\begin{proof}
We can expand $\|\calX^{(0)}-\calX^\star\|_F$ as
\begin{eqnarray}
    \label{The proof the spectral initialization in the Tucker format}
    &\!\!\!\!\!\!\!\!&\|\calX^{(0)}-\calX^\star\|_F=\|\text{HOSVD}(\frac{1}{m}\sum_{k=1}^my_k\calA_k)-\calX^\star\|_{F,2\ol r}\nonumber\\
    &\!\!\!\!\leq\!\!\!\!&\|\text{HOSVD}(\frac{1}{m}\sum_{k=1}^my_k\calA_k)-\frac{1}{m}\sum_{k=1}^my_k\calA_k\|_{F,2\ol r}
    +\|\frac{1}{m}\sum_{k=1}^my_k\calA_k-\calX^\star\|_{F,2\ol r}\nonumber\\
    &\!\!\!\!\leq\!\!\!\!&\sqrt{N}\|\text{opt}_{\vr^\text{tk}}(\frac{1}{m}\sum_{k=1}^my_k\calA_k)-\frac{1}{m}\sum_{k=1}^my_k\calA_k\|_{F,2\ol r}+
    \|\frac{1}{m}\sum_{k=1}^my_k\calA_k-\calX^\star\|_{F,2\ol r}\nonumber\\
    &\!\!\!\!\leq\!\!\!\!&(1+\sqrt{N})\|\frac{1}{m}\sum_{k=1}^my_k\calA_k-\calX^\star\|_{F,2\ol r}=(1+\sqrt{N})\max_{\calZ\in\R^{d_1\times\cdots\times d_N}, \|\calZ\|_F\leq 1, \atop \text{rank}(\calZ) = (2r_1^\text{tk},\dots,2r_{N}^\text{tk}) } |\<(\frac{1}{m}\calA^*\calA-\mathcal{I})(\calX^\star), \calZ \>|\nonumber\\
    &\!\!\!\!\leq\!\!\!\!& \delta_{3\ol r^\text{tk}}(1+\sqrt{N})\|\calX^\star\|_F,
\end{eqnarray}
where the second inequality follows from the quasi-optimality of HOSVD projection \cite{hackbusch2012tensor}. Here, we use the notation $\text{opt}_{\vr^\text{tk}}(\frac{1}{m}\sum_{k=1}^my_k\calA_k)$ to denote the best approximation to $\frac{1}{m}\sum_{k=1}^my_k\calA_k$ in the Frobenius norm.
Furthermore, due to the definition of $\text{opt}_{\vr^\text{tk}}(\cdot)$, it is always true that $\|\text{opt}_{\vr^\text{tk}}(\frac{1}{m}\sum_{k=1}^my_k\calA_k)-\frac{1}{m}\sum_{k=1}^my_k\calA_k\|_F \leq \|\frac{1}{m}\sum_{k=1}^my_k\calA_k-\calX^\star\|_F$, as stated in the third inequality. In the last line, we make use of \eqref{RIP of tensor sensing another_3}.

\end{proof}


\end{document}